\documentclass[sn-mathphys,Numbered,iicol]{sn-jnl}

\usepackage{graphicx}%
\usepackage{multirow}%
\usepackage{amsmath,amssymb,amsfonts}%
\usepackage{amsthm}%
\usepackage{mathrsfs}%
\usepackage[title]{appendix}%
\usepackage{xcolor}%
\usepackage{textcomp}%
\usepackage{manyfoot}%
\usepackage{booktabs}%
\usepackage{algorithm}%
\usepackage{algorithmicx}%
\usepackage{algpseudocode}%
\usepackage{listings}%
\usepackage{subfigure}
\hypersetup{
	colorlinks=True,
	linkcolor=blue,
	filecolor=black,      
	urlcolor=black,
	citecolor=blue,
}

\usepackage{color}

\newtheorem{theorem}{Theorem}
\newtheorem{lemma}{Lemma}
\newtheorem{definition}{Definition}%

\newcommand{\ie}{\textit{i.e.}}
\newcommand{\eg}{\textit{e.g.}}
\newcommand{\resp}{\textit{resp.}}

\newcommand{\x}{\mathbf{x}}
\newcommand{\z}{\mathbf{z}}
\newcommand{\s}{\mathbf{s}}
\newcommand{\w}{\mathbf{w}}
\newcommand{\E}{\mathbb{E}}

\newcommand{\F}{\mathcal{F}}
\newcommand{\X}{\mathcal{X}}
\newcommand{\Y}{\mathcal{Y}}
\newcommand{\Yc}{\mathcal{Y}_{\rm com}}
\newcommand{\tP}{P^{\rm c}}

\begin{document}

\title[Adversarial Reweighting with $\alpha$-Power Maximization for Domain Adaptation]{Adversarial Reweighting with $\alpha$-Power Maximization for Domain Adaptation}


\author[1]{\sur{Xiang Gu}}\email{xianggu@stu.xjtu.edu.cn}

\author[1]{\sur{Xi Yu}}\email{ericayu@stu.xjtu.edu.cn}

\author[1]{\sur{Yan Yang}}\email{yangyan@xjtu.edu.cn}

\author*[1]{\sur{Jian Sun}}\email{jiansun@xjtu.edu.cn}

\author[1]{\sur{Zongben Xu}}\email{zbxu@xjtu.edu.cn}

\affil[1]{\orgdiv{School of Mathematics and Statistics}, \orgname{Xi'an Jiaotong University}, \orgaddress{\street{Xianning West Road}, \city{Xi'an}, \postcode{710049}, \state{Shaanxi}, \country{China}}}




\abstract{{The practical Domain Adaptation (DA) tasks, \eg, Partial DA (PDA), open-set DA, universal DA, and test-time adaptation, have gained increasing attention in the machine learning community.} 
In this paper, we propose a novel approach,  dubbed Adversarial Reweighting with $\alpha$-Power Maximization (ARPM), for PDA where the source domain contains private classes absent in target domain. 
In ARPM, we propose a novel adversarial reweighting model that adversarially learns to reweight source domain data to identify source-private class samples by assigning smaller weights to them, for mitigating potential negative transfer. Based on the adversarial reweighting, we train the transferable recognition model on the reweighted source distribution to be able to classify common class data. To reduce the prediction uncertainty of the recognition model on the target domain for PDA, we present an $\alpha$-power maximization mechanism in ARPM, which enriches the family of losses for reducing the prediction uncertainty for PDA. Extensive experimental results on five PDA benchmarks, \ie, Office-31, Office-Home, VisDA-2017, ImageNet-Caltech, and DomainNet, show that our method is superior to recent PDA methods. Ablation studies also confirm the effectiveness of components in our approach. To theoretically analyze our method, we deduce an upper bound of target domain expected error for PDA, which is approximately minimized in our approach. { We further extend ARPM to open-set DA, universal DA, and test time adaptation, and verify the usefulness through experiments.}

}


\maketitle

\section{Introduction}\label{sec:introduction}
Deep learning approaches have achieved great success in visual recognition~\cite{He_2016_CVPR,krizhevsky2012imagenet,simonyan2014very}, but at the expense of laborious large-scale training data annotation. To alleviate the burden of data labeling, Domain Adaptation (DA) transfers the knowledge from a related but different source domain with rich labels to the label-scarce target domain. DA methods mainly learn the transferable model for the target domain by self-training~\cite{9733209,9512429,liu2021cycle,xu2022cdtrans} or by mitigating the domain shift using moment matching~\cite{9219132,koniusz2017domain,pmlr-v37-long15,sun2016deep,zellinger2017central} or adversarial training~\cite{pmlr-v37-ganin15,sun2016deep,tzeng2015simultaneous,tzeng2017adversarial}. 
{Conventional unsupervised DA is the closed-set DA setting, which assumes known target label space (identical to source label space) that is of the ``closed-world'' paradigm. However, it is often not easy to find a source domain with identical label space to the target domain in practice. Therefore, DA with label space mismatch, \eg, Partial Domain Adaptation (PDA)~\cite{cao2018partial,cao2018partial1}, open-set DA~\cite{panareda2017open,saito2018open}, and universal DA~\cite{you2019universal,fu2020learning},  has gained increasing attention in the machine learning community. PDA, open-set DA, and universal DA are related to the more realistic ``open-world'' paradigm. Specifically, open-world visual recognition does not assume a fixed set of categories (\ie, label space) as in closed-world visual recognition. For PDA, the target label space, being a subset of the source label space, is not fixed, because there exist numerous possible subsets of the source label space. For open-set DA and universal DA, the target domain contains unknown/open classes that are absent in the source domain. { Another practical DA setting is the Test-Time Adaptation (TTA)~\cite{wang2021tent}, allowing model adaptation at test time. This paper first focuses on the methodology design for PDA, and then extends the developed approach to open-set DA, universal DA, and TTA.} 
} 



PDA~\cite{cao2018partial,cao2018partial1,zhang2018importance,feng2019attract} tackles the setting that the source domain contains private classes absent in the target domain, while the target domain classes belong to the set of source domain classes. Besides the domain shift between source and target domains, another main challenge of PDA is the possible negative transfer~\cite{pan2009survey} {(see Sect.~\ref{sec:motivation})}, \ie, the knowledge from source domain harms the learning in the target domain, caused by the source-private class data. To mitigate the negative transfer, previous PDA methods~\cite{cao2018partial,cao2018partial1,cao2019learning,li2020deep,liang2020balanced,ren2020learning,yandiscriminative,zhang2018importance} commonly reweight the source domain data to decrease the importance of data belonging to the source-private classes. The target and reweighted source domain data are used to train the feature extractor by adversarial training~\cite{cao2018partial,cao2018partial1,cao2019learning,liang2020balanced,yandiscriminative,zhang2018importance} or kernel mean matching~\cite{li2020deep,ren2020learning} to align distributions in feature space.

In this paper, we propose a novel approach, dubbed Adversarial Reweighting with $\alpha$-Power Maximization (ARPM), for PDA. To alleviate the potential negative transfer caused by source-private class data, we propose an adversarial reweighting model to reweight the source domain data to decrease the importance of source-private class data in adaptation by assigning them with smaller weights. The learning of source data weights is conducted by minimizing the Wasserstein distance between the target distribution and the reweighted source distribution. The intuition is that the source domain common class data are possibly closer to the target domain data than the source-private class data. This is reasonable and is the assumption taken in~\cite{cao2019learning}, and otherwise, PDA could be hardly realized. Using the dual formulation of the Wasserstein distance, the idea is further transformed into an adversarial reweighting model, in which we introduce a discriminator to distinguish domains and adversarially learn the source data weights to fool the discriminator. 

Based on the reweighted source data distribution, we define a reweighted classification loss to train the model to recognize objects of common classes, in which the importance of source-private class data is reduced using the learned data weights. Inspired by~\cite{liang2020balanced} that bridges domain gap in feature space by entropy minimization~\cite{grandvalet2005semi} to reduce the prediction uncertainty\footnote{In this paper, by ``prediction uncertainty'', we refer to the uncertainty of the classification probability distribution (classification score) outputted by the recognition model, \eg, the uniform distribution has larger uncertainty while the one-hot distribution has smaller uncertainty.} of recognition model on target domain, we also aim to reduce the prediction uncertainty on target domain. Instead of entropy minimization, we propose an $\alpha$-power maximization mechanism that maximizes the sum of $\alpha$-power of the classification score outputted by the recognition model. The $\alpha$-power maximization enriches the family of losses for minimizing the prediction uncertainty. We experimentally show that the $\alpha$-power maximization could be more effective for PDA than the widely adopted entropy minimization~\cite{grandvalet2005semi}. We also utilize the neighborhood reciprocity clustering~\cite{yang2021exploiting}, which is shown to be effective for closed-set DA, to enforce the robustness of the recognition model for PDA.

The above techniques are unified in our total training loss. To train the recognition model, we design an iterative training algorithm that alternately updates the parameters of the recognition model and learns the source domain data weights by solving the adversarial reweighting model. To evaluate our proposed method, we apply our approach to the PDA tasks on five benchmark datasets: Office-31, Office-Home, VisDA-2017, ImageNet-Caltech, and DomainNet. On all five datasets, our proposed ARPM outperforms the recent PDA methods. Ablation studies and empirical analysis also show the effectiveness of each component in our method. 

To further theoretically analyze our method, we study the theoretical analysis of PDA from the perspective of robustness and prediction uncertainty of the recognition model. More specifically, we prove theoretically that the expected error of the recognition model on target domain can be bounded by the expected error on source domain common class data, and the robustness and prediction uncertainty on target domain of the recognition model. Our approach approximately realizes the minimization of this bound so as to minimize the expected error on target domain.

{Additionally, we extend our approach to open-set DA, universal DA, and TTA}. For open-set DA, the target domain contains private classes that are absent in the source domain. For universal DA, both source and target domains possibly contain private classes. The goals of open-set and universal DA are to identify the target-private class data as the ``unknown'' class and meanwhile classify the target domain common class data. To extend our approach to open-set and universal DA, we apply our adversarial reweighting model to reweight target domain data, such that the target domain common (\resp, private) data are assigned with larger (\resp, smaller) weights. Based on learned weights, we reduce (\resp, increase) the prediction uncertainty of target domain possibly common (\resp, private) data using our $\alpha$-power loss. As a result, the target-private class can be identified based on prediction uncertainty. {For TTA, the goal is to evaluate the model on a target domain that may be different from the source domain in data distribution.} Different from vanilla machine learning which directly makes predictions for a mini-batch of test samples at test time, TTA allows adapting the model for a few steps on the mini-batch of test samples in an unsupervised manner and then makes predictions for them. Inspired by the TTA method~\cite{wang2021tent} that updates the parameters of the batch normalization (BN) layers by entropy minimization for one step, we update the parameters of the BN layers by our proposed $\alpha$-power maximization for one step to achieve the extension to TTA. Experiments show the usefulness of our approach for open-set DA, universal DA, and TTA.

Our contributions are summarized as follows:
\begin{itemize}
    \item We propose a novel ARPM approach for PDA. In ARPM, we propose an adversarial reweighting model for learning to reweight source domain data to decrease the importance of source-private class data in adaptation for PDA. We also propose the $\alpha$-power maximization mechanism to reduce the prediction uncertainty.
    \item We present a theoretical bound of PDA based on the robustness and prediction uncertainty. We analyze that our proposed ARPM can realize the minimization of the bound. 
    \item Extensive experimental results show the superiority of ARPM against recent PDA methods, along with sufficient ablation studies verifying the effectiveness of each component.
    \item {We extend our approach to open-set DA, universal DA, and TTA that are closely related to ``open-world vision recognition''.}
\end{itemize}

This paper extends our conference version~\cite{gu2021adversarial} published at NeurIPS, in which we devised the adversarial reweighting model and reduced the prediction uncertainty by entropy minimization for PDA. In this journal version, we make the following additional contributions. (1) We propose to maximize the sum of $\alpha$-power of the classification score outputted by the recognition model, enriching the family of losses to minimize the prediction uncertainty. We experimentally show that the $\alpha$-power maximization could be more effective for PDA than the widely adopted entropy minimization. (2) We present a theoretical analysis for PDA based on the prediction uncertainty and robustness of the recognition model, which theoretically grounds our approach. (3) We also enhance the robustness of the recognition model using neighborhood reciprocity clustering. (4) To ensure reproducibility, more techniques, \eg, spectral normalization to the discriminator and initializing the classification layer using PCA, are introduced. (5) The performance of our approach is further improved compared with the conference version, and the journal version of our method outperforms recent PDA methods. (6) {We extend our approach to other ``open-world'' tasks, including open-set DA, universal DA, and TTA.}
(7) More related works are included and summarized. (8) The paper is restructured and rewritten to include the above contributions better.

In the following sections, we summarize the related works in Sect.~\ref{sec:related_work}, elaborate our ARPM approach in Sect.~\ref{sec:algorithm}, and present the theoretical analysis in Sect.~\ref{sec:theory}. Section~\ref{sec:experiment} discusses the experimental results. 

\section{Related Work}\label{sec:related_work}
{We summarize the related closed-world DA, PDA, open-set DA, and universal DA approaches below.}

\vspace{0.5\baselineskip} \noindent \textbf{Closed-world domain adaptation.}
Unsupervised DA~\cite{pan2009survey} aims to transfer knowledge learned from the labeled source domain to the unlabeled target domain, which generally assumes that the source and target domains share the same label space. A group of unsupervised DA methods~\cite{gretton2006kernel,tzeng2014deep,pmlr-v37-long15,sun2016deep,zellinger2017central,tang2023source} attempt to reduce the distribution gap between the source and target domains by moment matching. 
Another line of methods~\cite{pmlr-v37-ganin15,NIPS2018_7436,ganin2016domain,hoffman2018cycada,liu2021adversarial,du2021cross,li2021unsupervised,gu2020spherical} alleviate the domain discrepancy by introducing a domain discriminator to discriminate domains and training the feature extractor to fool the discriminator in an adversarial manner, to learn domain-invariant features. {Recently, DA approaches based on self-training~\cite{liu2021cycle,9512429,xu2022cdtrans}, transferable attention~\cite{wang2019transferable}, neighborhood consistency~\cite{yang2021exploiting,yang2021generalized}, progressive adaptation~\cite{li2022unsupervised}, and other techniques~\cite{9983498,9397282} without domain alignment have been proposed, achieving promising results.}
Differently, we mainly tackle the PDA setting by proposing a novel adversarial reweighting with $\alpha$-power minimization method for PDA. 
In methodology, our method may be mostly related to the adversarial training-based methods~\cite{pmlr-v37-ganin15,NIPS2018_7436,ganin2016domain,hoffman2018cycada,liu2021adversarial,du2021cross}. 
Different from them, our adversarial reweighting model conducts the adversarial training by learning to reweight data to fool the discriminator instead of training the feature extractor as in these methods. 

\begin{figure*}
	\centering
	\includegraphics[width=2.0\columnwidth]{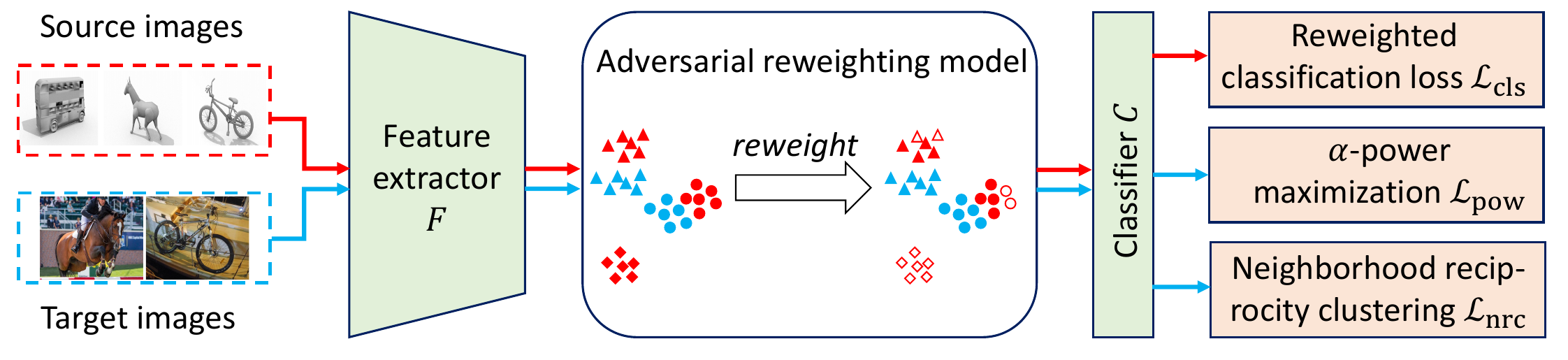}
	\caption{Architecture of our ARPM. Red (\textit{resp.} blue) arrows indicate the computational flow for source (\textit{resp.} target) domain data. Both source and target images are mapped to feature space by the feature extractor. Our adversarial reweighting model automatically reweights the importance of source domain data to match the target domain distribution in feature space to decrease the importance of the data of source-private class data. 
    We then define a reweighted classification loss on the reweighted source domain data distribution to train the recognition model to classify common class data. An $\alpha$-power maximization is proposed to reduce the prediction uncertainty on the target domain. We also utilize neighborhood reciprocity clustering~\cite{yang2021exploiting} to impose the robustness of the recognition model on the target domain.
}
	\label{fig:Architecture}
\end{figure*}

\vspace{0.5\baselineskip} \noindent \textbf{Partial domain adaptation.}  SAN~\cite{cao2018partial} and IWAN~\cite{zhang2018importance} circumvent negative transfer by training the domain discriminator with reweighted source domain samples. PADA~\cite{cao2018partial1} and ETN~\cite{cao2019learning} reweight source domain data in losses for training both the classifier and the domain discriminator. 
DRCN~\cite{li2020deep} uses reweighted class-wise domain alignment with a plug-in residual block that automatically uncovers the most relevant source domain classes to target domain data. TSCDA~\cite{ren2020learning} introduces a soft-weighed maximization mean discrepancy criterion to partially align feature distributions to alleviate negative transfer, and proposes a target-specific classifier to further address the classifier shift. SLM~\cite{sahoo2021select} exploits ``select", ``label" and ``mix" modules to mitigate negative transfer, enhance discriminability of features, and learn domain-invariant latent space, respectively.
ISRA~\cite{xiao2021implicit} aligns the source and target data distributions based on the implicit semantic topics shared between two domains that are extracted by a plug-in module, to boost the positive transfer. {IDSP~\cite{9983498} introduces intra-domain structure preserving without domain alignment, achieving improved results.}
MOT~\cite{Luo_2023_CVPR} utilizes a masked optimal transport on conditional distribution by defining the mask using label information to align class-wise distributions for PDA. Different from the above PDA methods, we adversarially learn to reweight the source domain data to decrease the importance of source-private class data in the classification loss, and propose $\alpha$-power maximization to reduce the prediction uncertainty.

 {
    \vspace{0.2\baselineskip}
    \noindent\textbf{Open-set domain adaptation.} 
    Busto and Gall~\cite{panareda2017open} first study the setting that both the source and target domains contain private categories in addition to the common categories. Saito~\textit{et al.}~\cite{saito2018open} consider the problem setting that the source domain only covers a subset of the target domain label space,  which is a common setting in the other open-set DA methods~\cite{saito2018open,liu2019separate,feng2019attract,baktashmotlagh2019learning,kundu2020towards,bucci2020effectiveness,rakshit2020multi,luo2020progressive,pan2020exploring,jing2021balanced,jing2021towards}. STA~\cite{liu2019separate} adopts a coarse-to-fine weighting mechanism to progressively separate the target domain data into known and unknown classes. Feng~\textit{et al.}~\cite{feng2019attract} exploit the semantic structure of open-set data by semantic categorical alignment and contrastive mapping to encourage the known classes more separable and push the unknown class away from the decision boundary. Baktashmotlag~\textit{et al.}~\cite{baktashmotlagh2019learning} tackle the open-set DA problem with a method based on subspace learning that models the common classes by a shared subspace and the unknown classes by a private subspace. 
    The method in~\cite{luo2020progressive} introduces a graph learning-based adversarial training strategy to align the known class samples from target domain with samples from source domain.
    Pan~\textit{et al.}~\cite{pan2020exploring} augment Self-Ensembling for both closed-set and open-set DA scenarios by integrating category-agnostic clusters into DA procedure. 
    Bucci~\textit{et al.}~\cite{bucci2020effectiveness} utilize a new open-set metric that properly balances the contribution of recognizing the known classes and rejecting the unknown samples, and investigate the self-supervised task of rotation recognition for facilitating open-set DA.
    Jing~\textit{et al.}~\cite{jing2021towards} develop structure-preserving partial alignment to recognize the seen categories and discover the unknown classes. 
    ANNA~\cite{li2023adjustment} tackles Open-set DA utilizing front-door adjustment theory.
    Different from the above methods, we propose the adversarial reweighting model to identify target-private class data for open-set DA. We further respectively decrease and increase the prediction uncertainty of the recognition model on target domain common and private class data, to classify/detect the common/private class data.

    \vspace{0.2\baselineskip}
    \noindent\textbf{Universal domain adaptation.}
    UAN~\cite{you2019universal} proposes a criterion to quantify sample-level transferability based on entropy and domain similarity, thereby promoting the adaptation in the automatically discovered common label set and recognizing the “unknown” samples successfully.
    CMU~\cite{fu2020learning} designs a better criterion based on a mixture of entropy, confidence, and consistency from a multi-classifier ensemble model to measure sample-level transferability.
    DANCE~\cite{saito2020universal} 
    uses entropy-based feature alignment and rejection to align target domain features with the source domain or reject the target domain features as unknown categories based on their entropy.
    OVANet~\cite{saito2021ovanet} introduces one-vs-all classifiers for each class to automatically learn the threshold for identifying the unknown class data.
    DCC~\cite{li2021domain} proposes a cluster-based method to exploit the domain consensus knowledge to discover discriminative clusters for separating the private classes from the common ones in target domain.
    The method in \cite{chen2022geometric} explores the intrinsic geometrical relationship between the two domains and designs a universal incremental classifier to separate “unknown” samples. Motivated by Bag-of-visual-Words, the method in \cite{kundu2022subsidiary} introduces subsidiary prototype-space alignment to tackle universal DA, avoiding negative transfer.
    GLC~\cite{sanqing2023GLC} introduces a global and local clustering learning technique for source-free universal DA. PPOT~\cite{Yang_Gu_Sun_2023} tackle universal DA
    based on a proposed prototypical partial optimal transport model to identify private class data.
    SAKA~\cite{wang2023exploiting} introduces knowability-guided detection of known and unknown samples and refines target pseudo labels based on neighborhood consistency. 
    Differently, we propose the novel adversarial reweighting model to reweight data for identifying the private class data of target domain, and perform domain adaptation relying on reducing/increasing prediction uncertainty based on the learned data weights.
    }

\section{Method}\label{sec:algorithm}

PDA assumes two related but different distributions, namely source distribution $P$ over space $\mathcal{X}\times\mathcal{Y}$ and the target distribution $Q$ over space $\mathcal{X}\times\mathcal{Y}_{\rm com}$, where $\mathcal{Y}_{\rm com}\subset\mathcal{Y}$. In training, we are given labeled source samples $S=\{(\x_i^s,y_i)\}_{i=1}^m$ drawn \textit{i.i.d.} from $P$, and unlabeled target samples $T=\{\x_j^t\}_{j=1}^n$ drawn \textit{i.i.d.} from $Q_{\x}$ where $Q_{\x}$ is the marginal distribution of $Q$ in space $\X$. The goal of PDA is to train a recognition model using the training samples to predict the class labels of target samples. $\mathcal{Y}_{\rm com}\subset\mathcal{Y}$ implies that the source domain contains private classes absent in the target domain, which may cause negative transfer in adaptation ({see Sect.~\ref{sec:motivation}}).
In this paper, we implement the recognition model using deep neural networks. Specifically, the recognition model is composed of a feature extractor $F$ (\eg, ResNet~\cite{He_2016_CVPR}) and a classifier $C$. Detailed architectures of $F$ and $C$ will be given in Sect.~\ref{sec:network}.

To tackle the PDA task, we propose a novel approach, dubbed Adversarial Reweighting with $\alpha$-Power Maximization (ARPM). The overall framework of ARPM is illustrated in Fig.~\ref{fig:Architecture}. We apply the feature to the input images to extract features for both source and target domains. In the feature space, we propose the adversarial reweighting model to reweight source features such that the source-private class features are assigned smaller weights. We then perform PDA based on the adversarial reweighting. Specifically, on the reweighted source domain data distribution, we define a reweighted classification loss to train the recognition model to be able to classify common class data. On the target domain data, we propose an $\alpha$-power maximization mechanism to reduce the prediction uncertainty of the recognition model. We also utilize the neighborhood reciprocity clustering~\cite{yang2021exploiting} to enforce the robustness of the recognition model on target domain data. {We next discuss our intuitive motivation in Sect.~\ref{sec:motivation}}, the adversarial reweighting model in Sect.~\ref{sec:adv_rew}, and adaptation based on adversarial reweighting in Sect.~\ref{sec:adapt}, in which we introduce the reweighted classification loss, the $\alpha$-power maximization, and the neighborhood reciprocity clustering, followed by our training algorithm in Sect.~\ref{sec:training}. {Finally, we extend our method to open-set DA and universal DA in Sect.~\ref{sec:extend_DA} and to TTA in Sect.~\ref{sec:extend_TTA}.}
{
\begin{figure}[t]
    \centering
    \includegraphics[width=0.9\columnwidth]{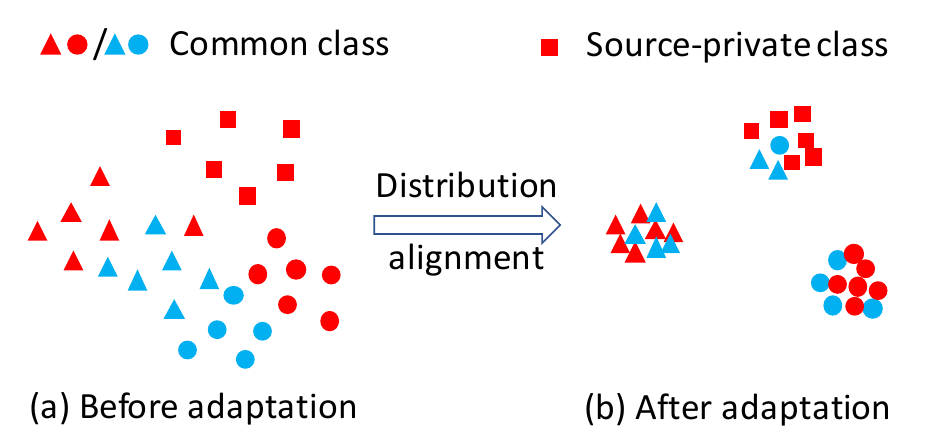}
     
    \caption{Illustration of negative transfer caused by the source-private class data in PDA. The source and target features are respectively in red and blue. Some of the target domain samples are unavoidably aligned with the source-private class data in feature adaptation by distribution alignment, and are incorrectly recognized by the recognition model.
    }
    \label{fig:negative_transfer}
\end{figure}
\begin{figure*}
    \centering
    \includegraphics[width=2.0\columnwidth]{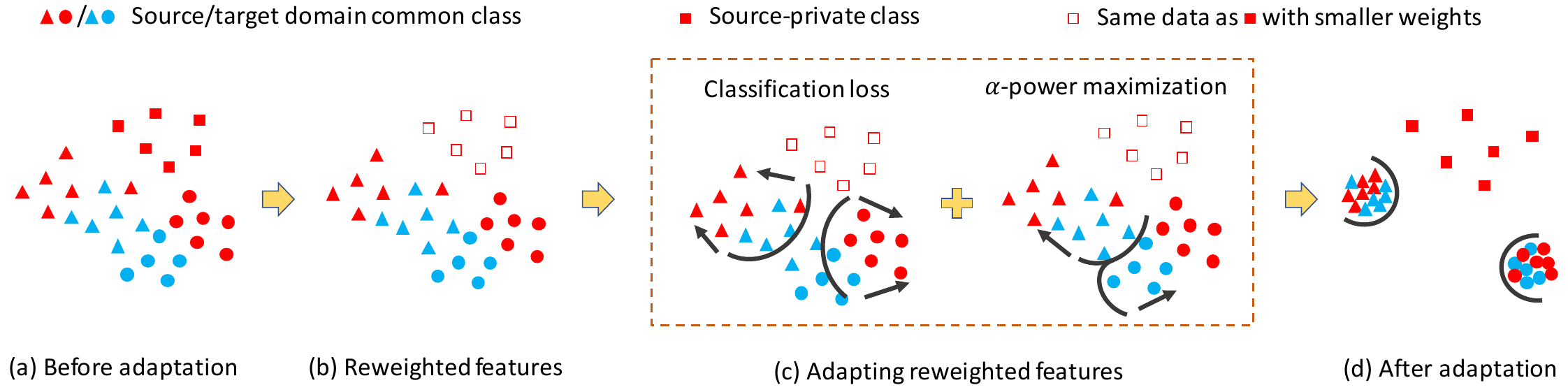}
    
    \caption{Intuitive motivations of ARPM.  We reweight source domain data by our adversarial reweighting model to assign smaller weights to source-private class data. The classification loss can enforce lower prediction uncertainty mainly on source domain common class data. We propose the $\alpha$-power maximization to lower prediction uncertainty on target samples. Intuitively, to achieve lower prediction uncertainty, the target samples will be pushed toward the regions of source domain common class data. 
    }
    \label{fig:intuitive_insights}
\end{figure*}

\subsection{Motivation}\label{sec:motivation}
We explain the negative transfer in PDA and the intuitive motivation of our method as follows.

\vspace{0.5\baselineskip} \noindent 
\textbf{Negative transfer.} The challenges of PDA arise from the distribution difference and the possible negative transfer caused by the source-private class data in adaptation. To enable the recognition model to be transferred from source to target domain, DA methods often align source and target distributions in feature space to adapt the feature extractor to tackle the challenge of distribution difference. However, some of the target domain data are unavoidably aligned with the source-private class data if directly aligning distributions, and thus are incorrectly recognized, as illustrated in Fig.~\ref{fig:negative_transfer}. In other words, the source-private class data can cause negative transfer when aligning distributions in PDA, \ie, these source-private class data harm the learning in the target domain.

\vspace{0.5\baselineskip} \noindent         
\textbf{Intuitive motivation.} Figure~\ref{fig:intuitive_insights} shows the motivations of our approach. Intuitively, our adversarial reweighting model aims to learn to reweight source domain data to assign smaller weights to source-private class data, as illustrated in Fig.~\ref{fig:intuitive_insights}(b). We then define the classification loss on the reweighted source domain data to train the recognition model, in which the source-private class data are less important because they are reweighted by smaller weights. The trained recognition model could be mainly discriminative on source domain common class data, \ie, the predictions are certain on these data. We propose the $\alpha$-power maximization to reduce the prediction uncertainty on the target domain, as will be discussed in Sect.~\ref{sec:adapt}. By using our $\alpha$-power maximization loss to train the feature extractor, the learned target features will be pushed toward the source common class features to achieve a lower prediction uncertainty, as shown in Fig.~\ref{fig:intuitive_insights}(c). After adaptation, the target features will be aligned with source domain common class features, as illustrated by Fig.~\ref{fig:intuitive_insights}(d) (also Fig.~\ref{fig:tsne_features} on real data). Therefore, the negative transfer could be alleviated in our approach.
}

\subsection{Adversarial Reweighting}\label{sec:adv_rew}
We follow~\cite{cao2019learning} to assume that the source domain data of common classes $\Yc$ are closer to the target domain data than the source domain data belonging to the source-private classes $\mathcal{Y}\backslash\mathcal{Y}_c$. This is reasonable, and otherwise, PDA could be hardly realized.  We then learn the weights of source domain data by minimizing the Wasserstein distance between the reweighted source and target distributions. The weight learning process is formulated as an adversarial reweighting model. Figure~\ref{fig:adv_idea} illustrates our idea. We first introduce the Wasserstein distance.

\vspace{0.5\baselineskip} \noindent \textbf{Wasserstein distance.} 
The Wasserstein distance is a metric from optimal transport that measures the discrepancy between two distributions. The Wasserstein distance between distributions $\mu$ and $\nu$ is defined by $W(\mu,\nu)=\min_{\pi\in\Pi}\mathbb{E}_{(\x,\x')\sim\pi}\left\|\x-\x'\right\|$, where $\Pi$ is the set of couplings of $\mu$ and $\nu$, \textit{i.e.}, $\Pi=\{\pi|\int\pi(\x,\x'){\rm d}\x'=\mu(\x),\int\pi(\x,\x'){\rm d}\x=\nu(\x')\}$, and $\left\|\cdot\right\|$ is the $l_2$-norm. Leveraging the Kantorovich-Rubinstein duality, the Wasserstein distance has the dual form of $W(\mu,\nu)=\max_{\left\|g\right\|_L\leq 1} \mathbb{E}_{\x\sim\mu}g(\x)-\mathbb{E}_{\x'\sim\nu}g(\x')$,
where the maximization is over all 1-Lipschitz functions $g:\mathbb{R}^d\rightarrow\mathbb{R}$. To compute the Wasserstein distance, we parameterize $g$ by a neural network $D$ (called discriminator). 
Then, the Wasserstein distance becomes
\begin{equation}\label{eq:Wasserstein_dist_dual}
W(\mu,\nu) \approx \max_{\left\|D\right\|_L\leq 1} \mathbb{E}_{\x\sim\mu}D(\x)-\mathbb{E}_{\x'\sim\nu}D(\x').
\end{equation}
In the conference version~\cite{gu2020spherical}, we enforce the constraint in Eq.~\eqref{eq:Wasserstein_dist_dual} with the gradient penalty technique as in~\cite{NIPS2017_892c3b1c}, which adds in a regularization term $-\beta\mathbb{E}_{\tilde{\x}\sim \tilde{P}_{\mu,\nu}}(\left\|\nabla_{\tilde{\x}}D(\tilde{\x})\right\|-1)^2$ to the objective function in Eq.~\eqref{eq:Wasserstein_dist_dual}, and $D$ is unconstrained. $\tilde{P}_{\mu,\nu}$ denotes the samples uniformly along lines between pairs of points sampled from distributions $\mu$ and $\nu$, \ie, $\tilde{\x}$ is constructed by $\tilde{\x}=\tau \x + (1-\tau)\x'\mbox{ where }\x\sim \mu, \x'\sim \nu, \tau\sim\mathcal{U}(0,1)$. However, this strategy introduces additional hyper-parameter $\beta$ and additional randomness from $\tau$, making the results less reproducible. In this journal version, we implement the Lipschitz constraint using spectral normalization~\cite{miyato2018spectral}, which normalizes the weight of each layer in $D$ by its spectral norm. We experimentally show in Sect.~\ref{sec:analysis} that the spectral normalization results in better reproducibility of our method than the gradient penalty.
Equation~\eqref{eq:Wasserstein_dist_dual} allows us to approximately compute the Wasserstein distance using gradient-based optimization algorithms on large-scale datasets. 
Compared with the other popular statistical distances, \textit{e.g.}, the JS-divergence, the Wasserstein distance enjoys better continuity for learning distributions~\cite{arjovsky2017wasserstein}.

\begin{figure}
	\centering
	\includegraphics[width=1.0\columnwidth]{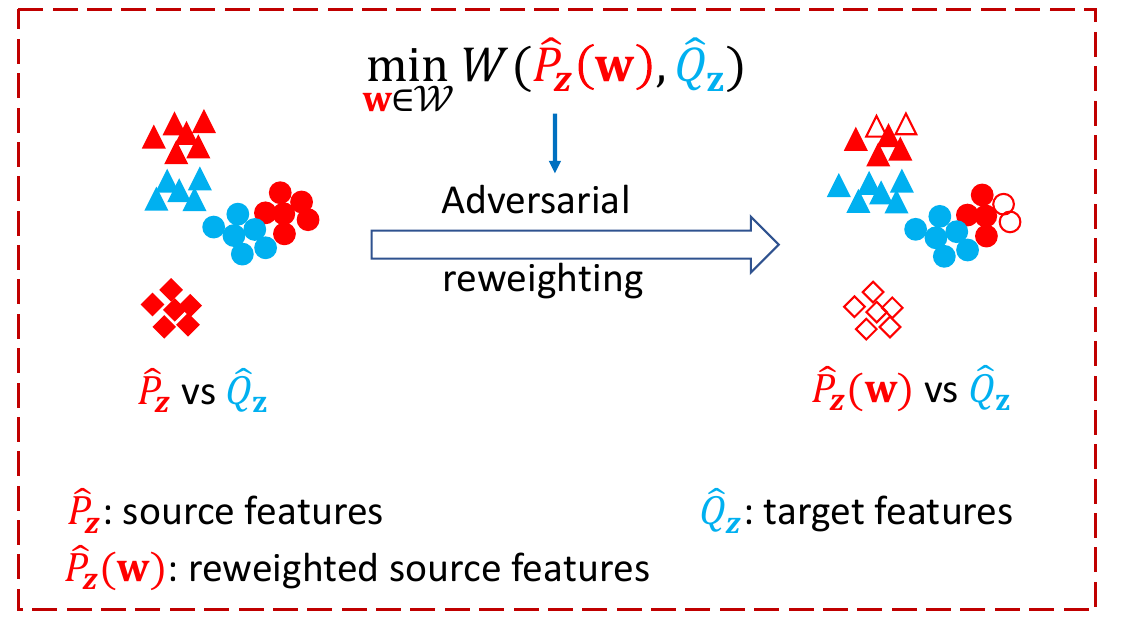}
	\caption{We minimize the Wasserstein distance between reweighted source feature distribution $\hat{P}_{\z}(\w)$ and target feature distribution $\hat{Q}_{\z}$ to learn weights $\w$. This idea is further transformed into the adversarial reweighting model.
}
	\label{fig:adv_idea}
\end{figure}
\subsubsection{Adversarial Reweighting Model}
We introduce source data weight $\frac{w_i}{m}$ (divided by $m$ is for the convenience of description) for each $i$, and denote $\mathbf{w}=(w_1,w_2,\cdot,w_m)$. Our adversarial reweighting model is defined in the feature space. We extract features by $\z_{i}^s=F(\x_i^s)$ and $\z_{j}^t=F(\x_j^t)$ for source and target domain data. 
The empirical distribution of the target domain features $\hat{Q}_{\z}$ is denoted as $\hat{Q}_{\z} = \frac{1}{n}\sum_{j=1}^{n}\delta_{\z_{j}^t}$.
The reweighted source domain feature distribution using $\mathbf{w}$ is denoted as $\hat{P}_{\z}(\mathbf{w}) = \frac{1}{m}\sum_{i=1}^{m}w_i\delta_{\z_{i}^s}$.
 Based on the aforementioned analysis that the source-private class data are more distant from target domain data than the source data of common classes, we minimize the Wasserstein distance between the reweighted source domain and target domain feature distributions to learn the weights (as illustrated in Fig.~\ref{fig:adv_idea}) as follows:
\begin{equation}\label{eq:weight_learning_model}
\min_{\mathbf{w}\in\mathbf{\mathcal{W}}} W(\hat{P}_{\z}(\mathbf{w}),\hat{Q}_{\z}).
\end{equation}
To avoid the mode collapse, \ie, the reweighted distribution is only supported on a few data, we enforce $\sum_{i=1}^{m}(w_i-1)^2 < \rho m$,
where $\rho$ is a hyper-parameter and is set to 5 in this paper. We will study the effect of $\rho$ in Sect.~\ref{sec:analysis}.
By this constraint, the difference between the learned data weights and the all-one vector (corresponding to unweighted data distribution) is not too large, avoiding the case that most samples are assigned with zero weight.
Then, the solution space is $\mathcal{W} = \{\mathbf{w}:\mathbf{w}=(w_1,w_2,\cdots, w_{m})^T,w_i\geq 0,  \sum_{i=1}^{m}w_i=m, \sum_{i=1}^{m}(w_i-1)^2 < \rho m\}$. 
With the approximation of the dual form in Eq.~\eqref{eq:Wasserstein_dist_dual}, Eq.~\eqref{eq:weight_learning_model} is transformed to the following adversarial reweighting model:
\begin{equation}\label{eq:adversarial_weight_learning_model}
\min_{\mathbf{w}\in\mathbf{\mathcal{W}}} \max_{\left\|D\right\|_L\leq 1} \frac{1}{m}\sum_{i=1}^{m}w_iD(\z_{i}^s) -\frac{1}{n}\sum_{j=1}^{n}D(\z_{j}^t).
\end{equation}
In Eq.~\eqref{eq:adversarial_weight_learning_model}, the discriminator $D$ is trained to maximize (\textit{resp.} minimize) the average of its outputs on the source (\textit{resp.} target) domain to discriminate the source and target domain data. Adversarially, the source data weights $\mathbf{w}$ are learned to minimize the reweighted average of the outputs of the discriminator on the source domain. As a result, the source data (closer to the target domain) with smaller discriminator outputs will be assigned with larger weights. We will discuss the adversarial training of Eq.~\eqref{eq:adversarial_weight_learning_model} in Sect.~\ref{sec:training}. 

\subsection{Adaptation Based on Adversarial Reweighting}\label{sec:adapt}
Based on the adversarial reweighting model, we perform PDA by defining a reweighted classification loss on reweighted source domain data, proposing an $\alpha$-power maximization mechanism to reduce prediction uncertainty on target domain, and utilizing the neighborhood reciprocity clustering to enforce the robustness.  We next discuss these three techniques.

\subsubsection{Reweighted Classification loss}
Based on the learned source domain data weights by the adversarial reweighting model, we define the reweighted classification loss on reweighted source domain data to implement the supervised training of the recognition model. The reweighted classification loss is defined using the cross-entropy by 
\begin{equation}\label{eq:class_loss}
    \mathcal{L}_{\rm cls}(F,C) = \frac{1}{m}\sum_{i=1}^{m}w_i \mathcal{J}(C(F(\x_i^s)),y_i^s).
\end{equation}
Following~\cite{liang2020we}, we employ the cross-entropy loss with label smoothing, \ie, $\mathcal{J}(\mathbf{p},y)=-\sum_k a_k$ $\log p_k$ for distribution $\mathbf{p}=(p_1,p_2,\cdots,$ $p_{|\Y|})^T$ where $a_k=1-\alpha$ if $k=y$, otherwise $a_k=\frac{\alpha}{|\Y|-1}$. $\alpha$ is set to 0.1.  
In the reweighted classification loss, the importance of the source-private class data is decreased because they are reweighted by smaller weights learned from the adversarial reweighting model. 
Minimizing the reweighted classification loss encourages the ability of the trained recognition model to classify common class data only.

\begin{figure*}[t]
    \centering
    \subfigure[$\mathcal{H}_{\alpha}(\mathbf{p}),\alpha=2$]{\includegraphics[width=0.48\columnwidth]{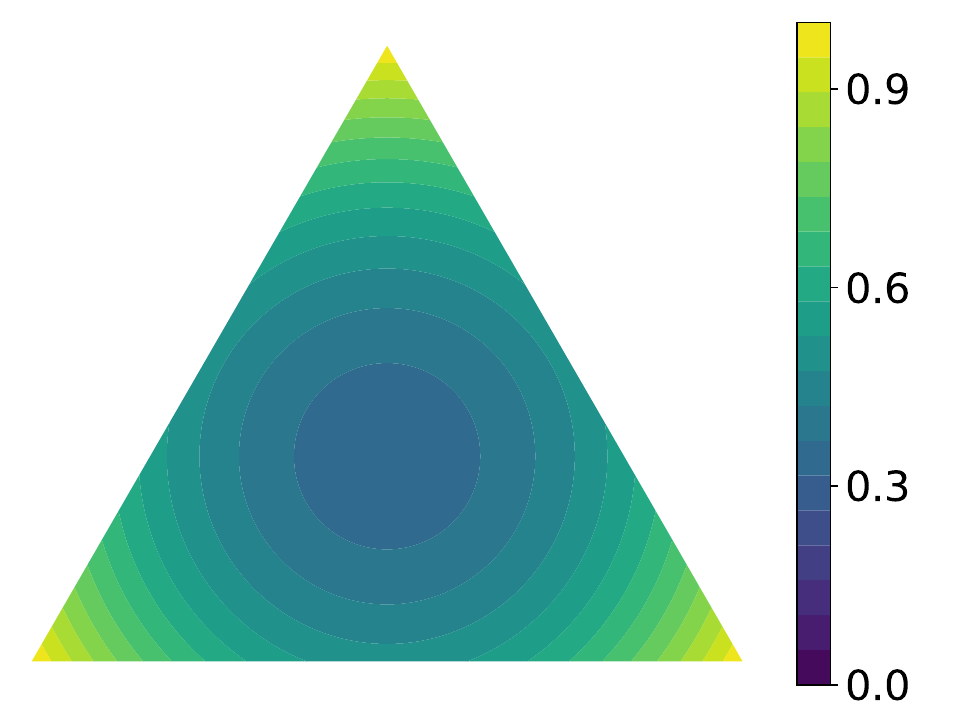}}
    \subfigure[$\mathcal{H}_{\alpha}(\mathbf{p}),\alpha=4$]{\includegraphics[width=0.48\columnwidth]{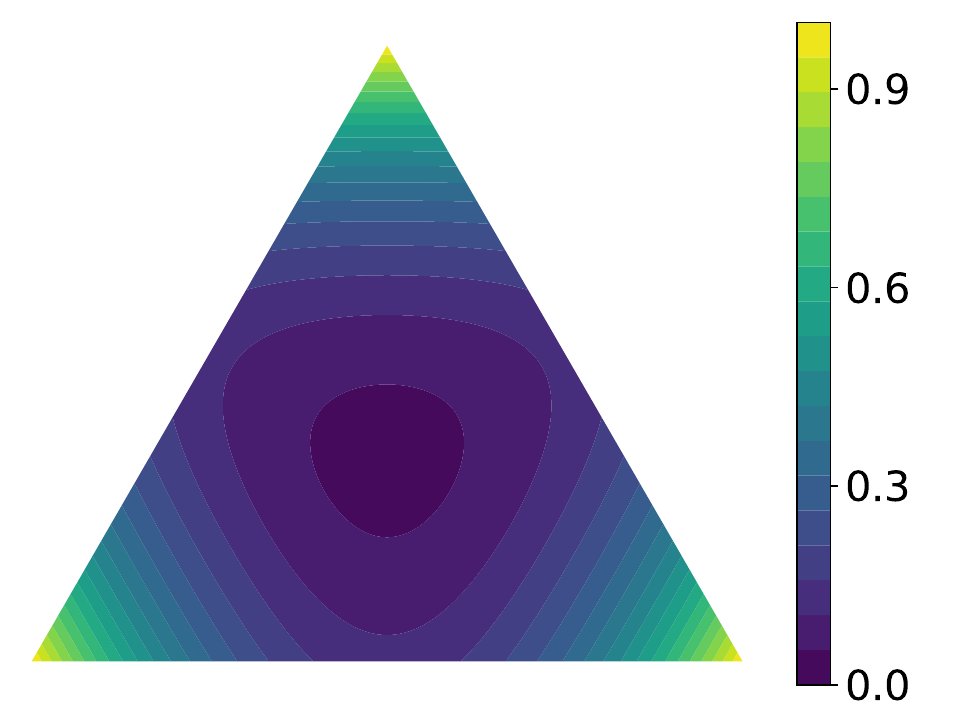}}
    \subfigure[$\mathcal{H}_{\alpha}(\mathbf{p}),\alpha=6$]{\includegraphics[width=0.48\columnwidth]{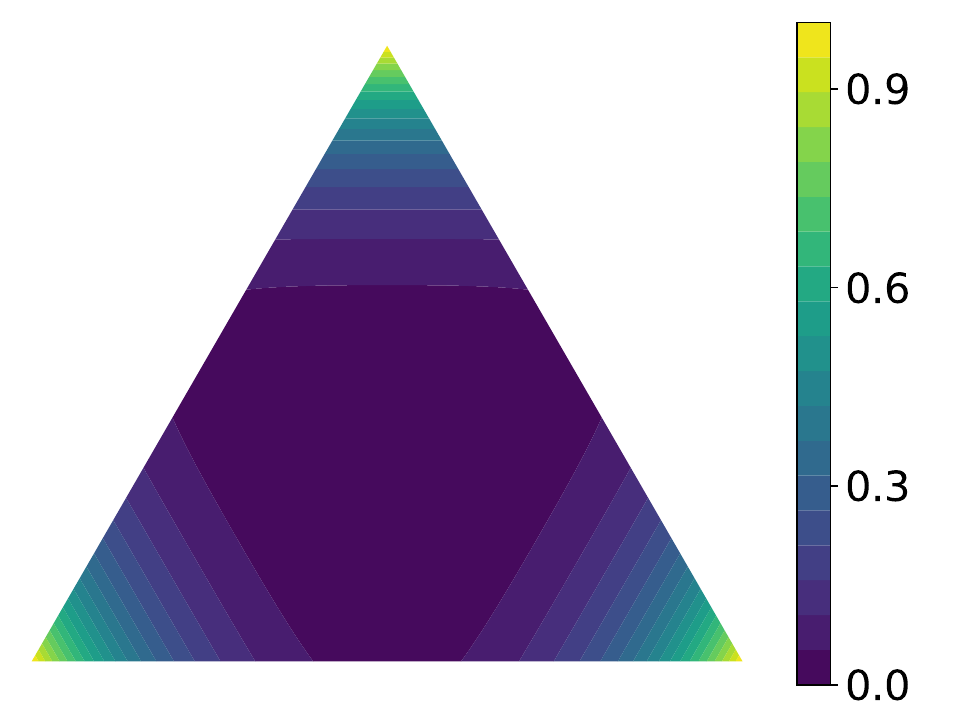}}
    \subfigure[$\mathcal{H}_{\alpha}(\mathbf{p}),\alpha=8$]{\includegraphics[width=0.48\columnwidth]{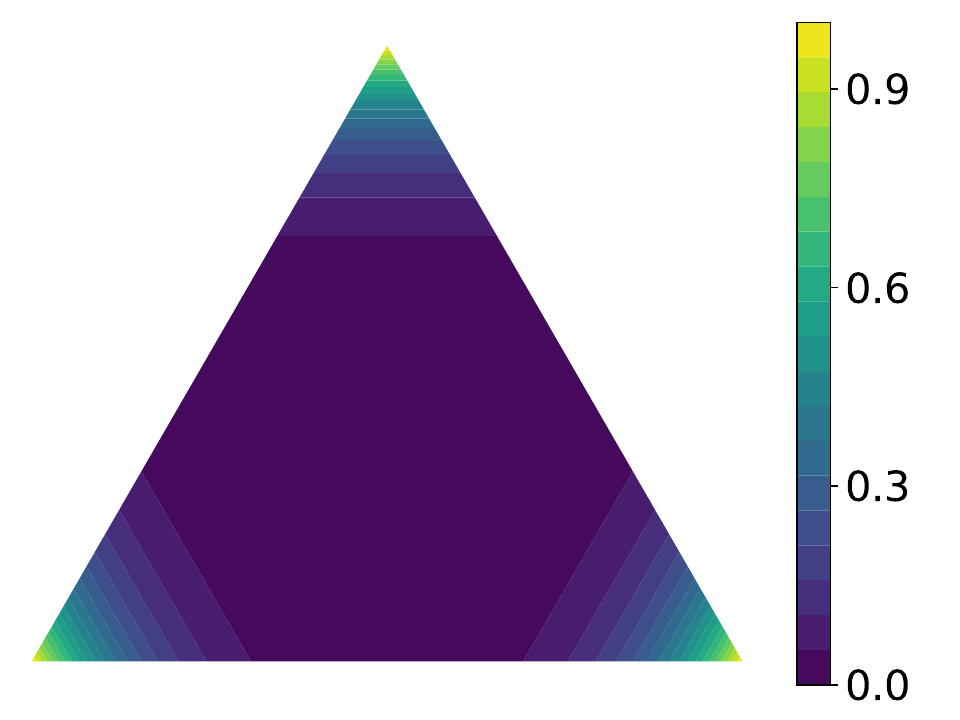}}
    \subfigure[$\|\nabla_{\mathbf{p}}\mathcal{H}_{\alpha}(\mathbf{p})\|,\alpha=2$]{\includegraphics[width=0.48\columnwidth]{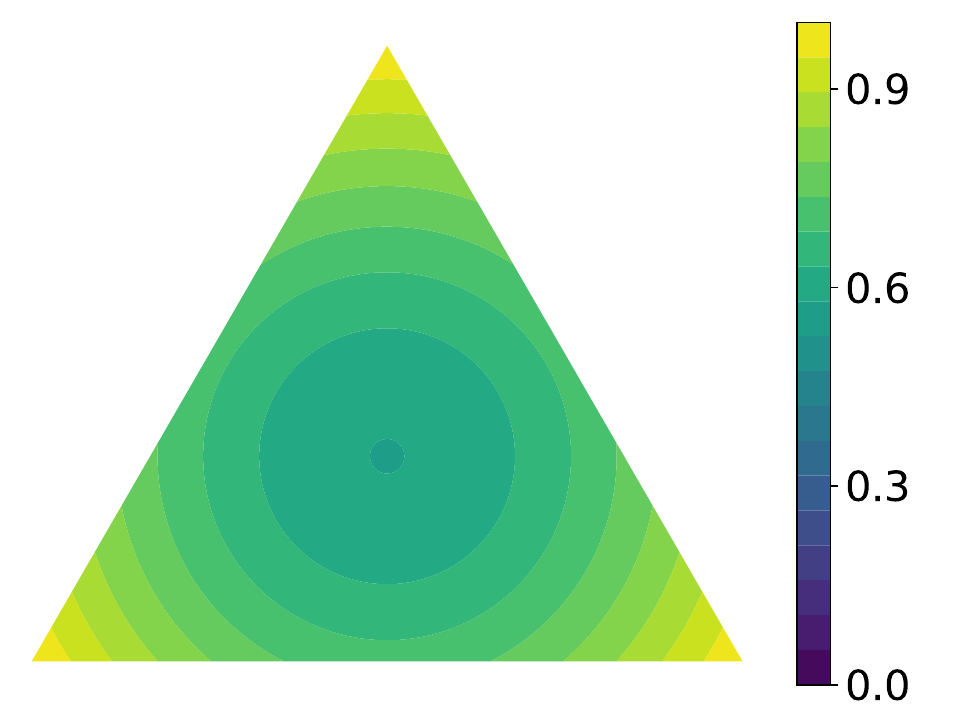}}
    \subfigure[$\|\nabla_{\mathbf{p}}\mathcal{H}_{\alpha}(\mathbf{p})\|,\alpha=4$]{\includegraphics[width=0.48\columnwidth]{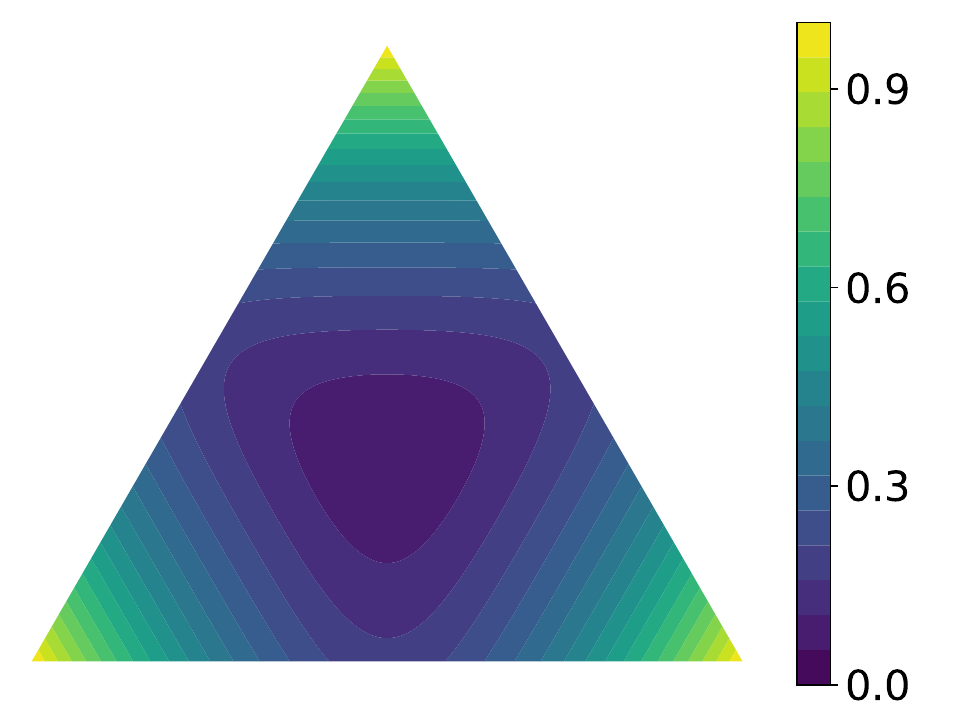}}
    \subfigure[$\|\nabla_{\mathbf{p}}\mathcal{H}_{\alpha}(\mathbf{p})\|,\alpha=6$]{\includegraphics[width=0.48\columnwidth]{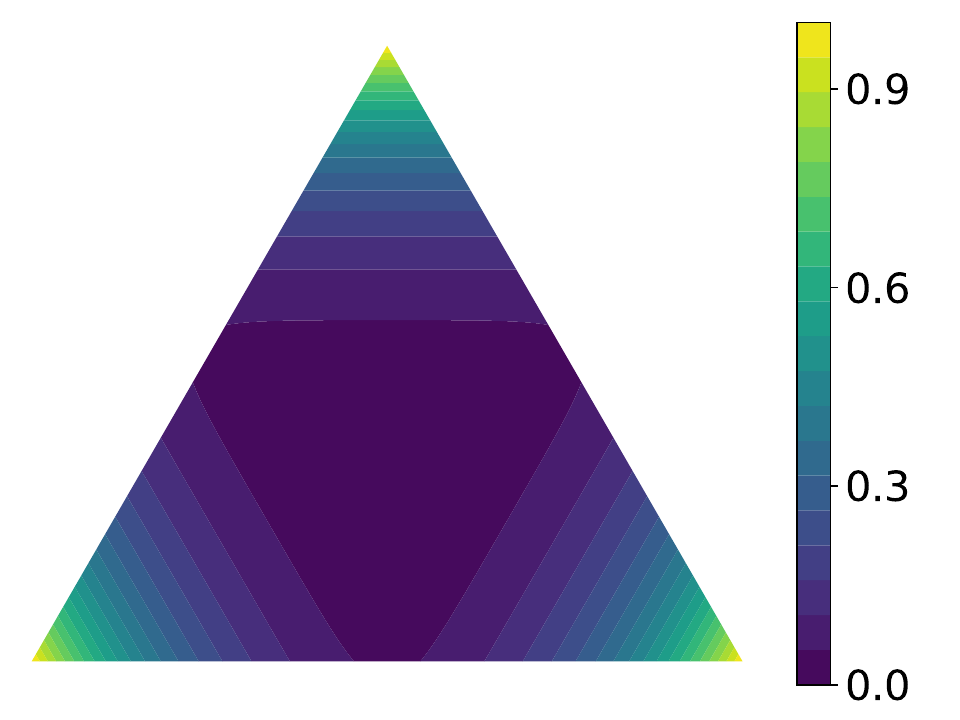}}
    \subfigure[$\|\nabla_{\mathbf{p}}\mathcal{H}_{\alpha}(\mathbf{p})\|,\alpha=8$]{\includegraphics[width=0.48\columnwidth]{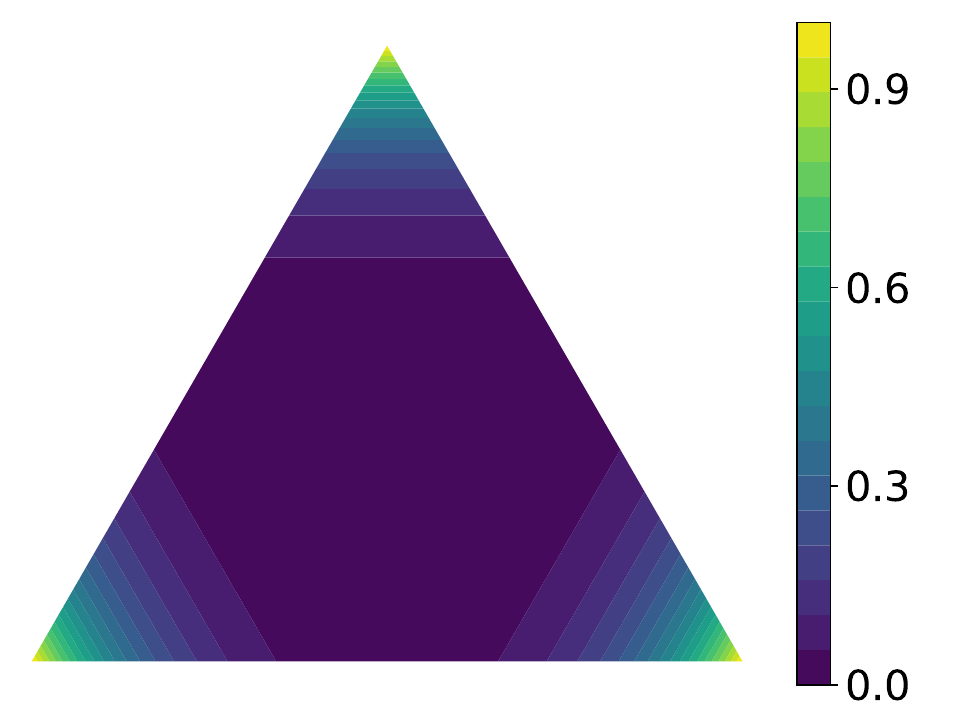}}
    \caption{Contour of (a-d) $\mathcal{H}_{\alpha}(\mathbf{p})$ and (e-h) gradient norm $\|\nabla_{\mathbf{p}}\mathcal{H}_{\alpha}(\mathbf{p})\|$ for 3-dimensional distribution $\mathbf{p}$ in probability simplex $\Delta^3$ under different $\alpha$. To show the relative magnitude for different $\mathbf{p}$, we normalize the gradient norm such that the maximum value is 1 in each figure of (e-h).}
    \label{fig:alpha_power}
\end{figure*}
\subsubsection{$\alpha$-Power Maximization}\label{sec:power_max}
Reducing the prediction uncertainty is shown to be effective in DA and even in PDA.
In our conference version~\cite{gu2021adversarial} of this work, we utilize entropy minimization that is widely adopted in DA methods to reduce the prediction uncertainty. In this journal paper, we propose to maximize the sum of $\alpha$-power of the output of the recognition model. The $\alpha$-power loss is defined by 
\begin{equation}
    \mathcal{L}_{\rm pow}(F) = -\frac{1}{n}\sum_{j=1}^{n}\mathcal{H}_{\alpha}(C(F(\x_j^t))),
\end{equation}
where $\mathcal{H}_{\alpha}(\mathbf{p})=\sum_k p_k^\alpha$. We empirically set $\alpha=6$ in experiments. $\mathcal{H}_{\alpha}(\mathbf{p})$ with $\alpha>1$ takes its maximum value when $\mathbf{p}$ is a one-hot distribution (with low uncertainty) and minimum value when $\mathbf{p}$ is a uniform distribution (with high uncertainty). Minimizing the $\alpha$-power loss (\ie, maximizing $\mathcal{H}_{\alpha}$) will reduce the uncertainty of $\mathbf{p}$. We plot the contour of $\mathcal{H}_{\alpha}(\mathbf{p})$ and its gradient norm $\|\nabla_{\mathbf{p}}\mathcal{H}_{\alpha}(\mathbf{p})\|$ under different $\alpha$ in Fig.~\ref{fig:alpha_power}. We can see that for larger $\alpha$, the samples with high prediction uncertainty (points near to the center in Figs.~\ref{fig:alpha_power}(a-h)) have smaller or even near-to-zero gradients of $\mathcal{H}_{\alpha}$, implying that these samples may not contribute to the training. The $\alpha$-power loss enriches the family of losses for reducing prediction uncertainty. Following~\cite{zhang2018importance}, we use the $\alpha$-power loss to update the feature extractor $F$. We show that $\alpha$-power maximization could be more effective for PDA than entropy minimization in Sect.~\ref{sec:analysis}.

\vspace{0.5\baselineskip} \noindent \textbf{Comparison of different uncertainty losses.} 
In the learning tasks with unlabeled data, \eg, semi-supervised learning and DA, reducing the uncertainty of model's prediction on unlabeled data in training can often improve the performance of the model~\cite{zhang2018importance,grandvalet2005semi}. The most used uncertainty loss in DA could be conditional entropy. Some methods~\cite{li2020deep,Cui_2020_CVPR} also investigate the mutual information and the nuclear norm for balancing the prediction over classes. \citet{liu2021cycle} minimize the $\alpha$-Tsallis entropy, \ie, $\frac{1}{\alpha-1}(1-\sum_k p_k^{\alpha})$, but choose $\alpha$ in a narrow interval $[1, 2]$, possibly limiting its ability. \citet{chen2019domain} propose to maximize the square loss of prediction probability, \ie, $\sum_k p_k^2$. Our $\alpha$-power loss is mostly related to the square loss and the $\alpha$-Tsallis entropy.
The square loss is a special case ($\alpha=2$) of the $\alpha$-power loss. Compared with the $\alpha$-Tsallis entropy, the $\alpha$-power loss is possibly more stable to $\alpha$ because of the presence of $\frac{1}{\alpha-1}$ in the $\alpha$-Tsallis entropy. We experimentally find that 
$2<\alpha\leq10$ often yields better results than $1<\alpha\leq2$ for PDA tasks, as in Sect.~\ref{sec:analysis}.

\subsubsection{Neighborhood Reciprocity Clustering}\label{sec:NRC}
The robustness or local consistency~\cite{yang2021exploiting} that enforces the outputs of the recognition model on the neighborhood samples are similar, is proven effective for closed-set DA. We utilize the neighborhood reciprocity clustering~\cite{yang2021exploiting} to impose the robustness for PDA in this paper.

We first build the feature and score banks on target domain data by
\begin{equation}
   \mathcal{Z}=\{\z^t_1,\z^t_2,\cdots,\z^t_n\}, \mathcal{S}=\{\s^t_1,\s^t_2,\cdots,\s^t_n\}
\end{equation}
where $\z^t_j=F(\x_j^t)$ is the feature, and $\s^t_j=C(F(\x_j^t))$ is the classification score outputted by the recognition model. At each training step, feature and score banks are updated by replacing the old items with the corresponding items from the current mini-batch samples. 

Given target sample $\x_j^t$ with feature $\z_j^t$, the index set of its $K$-nearest neighbors\footnote{Following \cite{yang2021exploiting}, we use the cosine distance to find the neighbors.} in the feature bank is denoted as $\mathcal{N}_K^j$. To better identify the true neighbors, we introduce the affinity\footnote{As in \cite{yang2021exploiting}, we set $K=M=5$ for VisDA-2017 dataset and $K=4$, $M=3$ for the other datasets in experiments.} as 
\begin{equation}\label{eq:aff}
    A_{j,j'}=\begin{cases}
        1 & \mbox{ if } j' \in \mathcal{N}_K^{j} \mbox{ and } j\in \mathcal{N}_M^{j'};\\
        0.1 & \mbox{ otherwise}.
    \end{cases}
\end{equation}
The intuition of the affinity in Eq.~\eqref{eq:aff} is that if $\z^t_j$ and $\z^t_{j'}$ are both neighbors of each other, they could be true neighbors (reciprocal neighbors). Finally, the neighborhood reciprocity clustering loss is defined as 
\begin{equation}\label{eq:nrc_loss}
    \mathcal{L}_{\rm nrc}(F) = -\frac{1}{n}\sum_{j=1}^n\sum_{j'\in\mathcal{N}_K^j}A_{j,j'}\langle\s_{j'}^t,C(F(\x_j^t))\rangle.
\end{equation}
We use $\mathcal{L}_{\rm nrc}(F)$ to update $F$. Minimizing $\mathcal{L}_{\rm nrc}(F)$ encourages the similar output of the neighborhood samples so that the robustness could be imposed. 

{Note that Zhong et al.~\cite{zhong2017re} proposed a $k$-reciprocal encoding method to re-rank the re-identification results for the person re-identification task, which may be related to the neighborhood reciprocity clustering in our approach. They encode the $k$-reciprocal nearest neighbors of a given image into a single vector, which is used for re-ranking under the Jaccard distance to help accomplish the person re-identification task. Differently, we utilize the reciprocal neighbors to enforce the robustness of the deep recognition model, instead of aggregating them as in~\cite{zhong2017re}.}

\begin{algorithm*}[t] 
	\caption{ Training algorithm.} 
	\label{alg:training} 
	{\bf Input:} Source and target domain training datasets $S$ and $T$\\
	{\bf Output:} Trained networks $F, C$
	\begin{algorithmic}[1]
	\State ${S}' = {S}$
	\State Initialize $\mathbf{w}$ by $w_i=1$, $i=1,2,\cdots,|S'|$
    \For{$step = 0, 1, 2, \cdots$}
    \State Sample a mini-batch data $(X_s,Y_s)$ and $X_t$  from ${S}'$ and ${T}$, respectively 
    \State Update $F$ and $C$ with the loss in Eq.~\eqref{eq:overall_loss} computed on $(X_s,Y_s)$ and $X_t$, using the SGD
    \If{$step \%N = 0$ and $step>0$}
    \If{$|{S}|>$1000k}
            \State Randomly select $N'$ samples from ${S}$ to construct ${S}'$
        \EndIf
        \State Extract features for all training samples in both ${S}'$ and ${T}$
        \State Train $D$  as in Eq.~\eqref{eq:Wasserstein_dist_dual} with spectral normalization to enforce the 1-Lipschitz constraint
        \If{$|S'|>$20k}
            \State Split source data indexes into $\left[\frac{m}{\mbox{20k}}\right]$ groups
            \State Solve Eq.~\eqref{eq:optimization_of_w} for each group to update $\w$
        \Else
            \State Solve Eq.~\eqref{eq:optimization_of_w} for all source data indexes to update $\w$
        \EndIf
    \EndIf
    \EndFor
	\end{algorithmic} 
\end{algorithm*}

\subsection{Training}\label{sec:training}
The overall training loss is
\begin{equation}\label{eq:overall_loss}
    \mathcal{L}  = \mathcal{L}_{\rm cls} + \lambda \mathcal{L}_{\rm pow} + \mathcal{L}_{\rm nrc},
\end{equation}
where $\lambda$ is a hyper-parameter. Note that in the classification loss $\mathcal{L}_{\rm cls}$ in Eq.~\eqref{eq:class_loss}, the source domain data weights $\w$ should be learned in training. We then devise an iterative training algorithm to alternately train $F$ and $C$, and learn $\w$.

\subsubsection{Training Algorithm}\label{sec:training_algorithm}
We initialize $\mathbf{w}$ by $w_i=1$ for all $i$.  Then, we alternately run the following two procedures when training the networks.

\vspace{0.5\baselineskip} \noindent \textbf{Updating ${F\mbox{ and }C}$ with fixed ${\mathbf{w}}$.}
Fixing $\mathbf{w}$, we update $F$ and $C$ to minimize the loss in Eq.~\eqref{eq:overall_loss} for $N$ steps, using the mini-batch stochastic gradient descent algorithm. 

\vspace{0.5\baselineskip} \noindent \textbf{Updating ${\mathbf{w}}$ with fixed ${F\mbox{ and } C}$.}
Fixing $F$ and $C$, we extract the features for all training data on both source and target domains, and learn $\mathbf{w}$ in Eq.~\eqref{eq:adversarial_weight_learning_model}.
Since Eq.~\eqref{eq:adversarial_weight_learning_model} is a min-max optimization problem, we can alternately optimize the weights $\mathbf{w}$ and the discriminator $D$ by fixing the other one as known. To reduce the computational cost, we only perform the alternate optimization once, which yields satisfactory performance in experiments. Therefore, we first fix $w_i=1$ for all $i$ and optimize $D$ to maximize the objective function in Eq.~\eqref{eq:adversarial_weight_learning_model}. Then, fixing the discriminator, we optimize $\mathbf{w}$ as follows. We denote $d_i = D(\z_{i}^s)$ and $\mathbf{d}=(d_1,d_2,\cdots,d_{m})^T$.  
The optimization problem for $\mathbf{w}$ becomes 
\begin{equation}\label{eq:optimization_of_w}
\begin{split}
\min_{\mathbf{w}} \hspace{0.1cm}&\mathbf{d}^T\mathbf{w} ,\\
\mbox{s.t.}\hspace{0.1cm}  &w_i\geq 0, \sum_{i=1}^{m}w_i =m,
\sum_{i=1}^{m}(w_i-1)^2\leq\rho m.
\end{split}
\end{equation}
Equation~\eqref{eq:optimization_of_w} is a second-order cone program.  We use the CVXPY~\cite{diamond2016cvxpy} package to solve Eq.~\eqref{eq:optimization_of_w}.

\vspace{0.5\baselineskip}
The above algorithm faces challenges when $m$ is large. 
We next discuss how to tackle the cases with larger size $m$ of the source dataset.
(1) In the first case (20k$<m<$1000k), solving Eq.~\eqref{eq:optimization_of_w} for all source data is time-consuming or even infeasible. For such a case, once the discriminator is trained and fixed, we split the source data indexes into several groups and solve the problem \eqref{eq:optimization_of_w} for each group. The number of groups is $l=\left[\frac{m}{\mbox{20k}}\right]$, where $[\cdot]$ is the floor function. The $i$-th group is $\{i,l+i,2l+i,\cdots\}$. Such a splitting strategy ensures that the empirical feature distribution corresponding to each group can approximate the empirical distribution of all features. 
(2) In the second case ($m>$1000k), extracting the features of all source samples is time-consuming. For such a case, we randomly sample a subset (with size $N'$) of the source dataset to learn their weights $\w$, and then update $F \mbox{ and } C$ on the subset and weights, and iterate these two above procedures.
We give the pseudo-code of the training algorithm in Algorithm~\ref{alg:training}.

\subsection{Network Details}\label{sec:network}
For the discriminator $D$, we use the same architecture as \cite{ pmlr-v37-ganin15} (three fully connected layers with 1024, 1024 and 1 nodes respectively), excluding the last sigmoid function.
For the feature extractor $F$, we use the ResNet-50~\cite{He_2016_CVPR} pre-trained on ImageNet~\cite{russakovsky2015imagenet}, excluding the last fully-connected layer. The classifier $C$ is a fully-connected layer. Inspired by~\cite{gu2020spherical,liang2020we}, we perform $L_2$-normalization after the feature extractor to enforce the same norm\footnote{We set the norm as in~\cite{gu2020spherical}.} of features, and normalize each row of the weight of classifier $C$ to a unit vector\footnote{On VisDA-2017 dataset, we do not normalize the weight of $C$. We empirically find that on VisDA dataset, the unnormalized weight of $C$ yields better result.}.

Since the features are normalized, the feature distribution is not a Gaussian distribution. Therefore, the commonly adopted initialization strategies~\cite{pmlr-v9-glorot10a,7410480} that assume the Gaussian distribution could not be suitable. 
While the initialization of the classifier can affect the reproducibility as we experimentally show in Sect.~\ref{sec:analysis}.
Following the idea of preserving the variance in~\cite{pmlr-v9-glorot10a,7410480}, we use Principal Component Analysis (PCA) to initialize the weight of the classifier to preserve the feature variance. Specifically, we compute principal components $V = (\mathbf{v}_1,\mathbf{v}_2,\cdots,\mathbf{v}_{|\Y|})$ of target features. We then compute the  principal component scores of each source feature and assign the source feature to the principal component with the largest score. Between the class label and the assigned principal component, we can calculate the confusion matrix $\mathcal{M}\in\mathbb{R}^{|\Y|\times|\Y|}$, of which the entry $\mathcal{M}_{ij}$ is the ratio of the $i$-th class source samples being assigned to the $j$-th principal component. Finally, the weight $W$ of $C$ is initialized by $W = \mathcal{M}V^T$. This PCA-based initialization strategy reduces the randomness compared with \cite{pmlr-v9-glorot10a,7410480}, and may achieve better reproducibility, as shown in Sect.~\ref{sec:analysis}.

{
\subsection{Extension to Open-set and Universal DA}\label{sec:extend_DA}
We extend our approach to open-set DA and universal DA in this section. In open-set DA, the unlabeled target domain contains private classes that are absent in the source domain. In universal DA, both the labeled source domain and unlabeled target domain possibly contain private classes. The goals of both open-set DA and universal DA are to identify the target-private class data as ``unknown'' and classify the target domain common class data. To extend our approach ARPM to open-set DA and universal DA, we employ our adversarial reweighting model to reweight target domain data, such that the target domain common (\resp, private) class data are assigned larger (\resp, smaller) weights. That is, in Eq.~\eqref{eq:adversarial_weight_learning_model}, we reweight $D(\mathbf{z}_j^t)$ by a weight $w_j$ and $D(\mathbf{z}_i^s)$ is unweighted. Based on the reweighted target domain data, we reduce (\resp, increase) the prediction uncertainty on the target domain common (\resp, private) class data so that the target-private class data can be detected using the prediction uncertainty.

More specifically, we sort the target samples according to the learned data weights in ascending order to obtain $\x^t_{(1)}, \x^t_{(2)}, \cdots, \x^t_{(n)}$. Since the target samples with larger weights possibly belong to common classes, we define the following reweighted $\alpha$-power loss to reduce the prediction uncertainty on these data:
\begin{equation}
    \mathcal{L}_{\rm com}(F) = -\frac{1}{n\tau}\sum_{l=n-n\tau}^n w_{(l)} \mathcal{H}_{\alpha}(C(F(\x_{(l)}^t))),
\end{equation}
where we discard the $n-n\tau$ samples with smaller weights to further enforce robustness. 
We also maximize the following $\alpha$-power loss on the $n\tau$ samples with smaller weights that are more possibly private classes, to increase their prediction uncertainty:
\begin{equation}
    \mathcal{L}_{\rm pri}(F) = -\frac{1}{n\tau}\sum_{l=1}^{n\tau} \mathcal{H}_{\alpha}(C(F(\x_{(l)}^t))).
\end{equation}
The total training loss for open-set DA and universal DA is 
\begin{equation}
    \mathcal{L} = \mathcal{L}_{\rm cls} + \lambda'(\mathcal{L}_{\rm com} - \mathcal{L}_{\rm pri}).
\end{equation}
$\mathcal{L}_{\rm cls}$ is defined in Eq.~(4) in the revised paper, in which we do not reweight the source domain data. For universal DA, one may also reweight the source domain data in $\mathcal{L}_{\rm cls}$ to reduce the importance of source-private class data. In this extension, for the sake of a unified formulation for open-set DA and universal DA, we do not reweight the source domain data.

\subsection{Extension to TTA}\label{sec:extend_TTA}
This section extends our approach to TTA. TTA trains the recognition model on a source domain and evaluates it on an unknown target domain. Different from vanilla machine learning which directly makes predictions for a mini-batch of test samples at test time, TTA allows adapting the model for a few steps on the mini-batch of test samples in an unsupervised manner and then makes predictions for them. At the test time of TTA, we are given the source-trained recognition model. The test data arrive sequentially. At each time, we only access a mini-batch of target data $B=\{\x_j^t\}_{j=1}^b$ where $b$ is the batch size. The goal of TTA is to update the model on the mini-batch data $B$ and then make predictions on them. Inspired by TENT~\cite{wang2021tent} that updates the parameters of the BN layers by entropy minimization for one step, we update the parameters of the BN layers by our proposed $\alpha$-power maximization for one step. Specifically, if we denote the set of parameters of the BN layers in the model as $\Theta_{\rm bn}$, the $\alpha$-power loss is defined as 
        \begin{equation}
            \mathcal{L}_{\rm pow}(\Theta_{\rm bn}) = \frac{1}{b}\sum_{j=1}^b \mathcal{H}_{\alpha}(C(F((\x_j^t))).
        \end{equation}
We update $\Theta_{\rm bn}$ using $\mathcal{L}_{\rm pow}$ by one-step gradient descent, and then predict the label of $B$. Note that the updated $\Theta_{\rm bn}$ will be taken as the initial value for the next mini batch. 
Our approach for TTA is dubbed Test $\alpha$-Power Maximization (TPM).
}

\section{Theoretical Analysis}\label{sec:theory}
This section presents the theoretical analysis of our method for PDA. We first provide the notations and assumptions,  then present an upper bound for PDA, and finally analyze that our proposed ARPM approximately realizes the minimization of the upper bound. 

\vspace{0.5\baselineskip} \noindent \textbf{Notations.}
We denote $P^c$ as the source domain common class data distribution, \ie, $\tP(\x,y)=P(\x,y|y\in\Y_{\rm com})$. For any $\x\in\X$, its neighborhood is defined by $\mathcal{N}(\x)=\{\x':d(\x,\x')\leq \xi\}$ for some $\xi>0$, where $d(\cdot,\cdot)$ is the distance on $\X$. For any set $A\subset \X$, we define $\mathcal{N}(A) = \bigcup_{\x\in A}\mathcal{N}(\x)$. For convenience, we denote $f:\X\rightarrow [0,1]^{|\Y|}$ as the recognition model, \ie, $f=C\circ F$. The decision function corresponding to $f$ is $\tilde{f}$ defined by $\tilde{f}(\x)=\arg\max_{i\in\Y}{f}(\x)_i$. We use  $\F$/$\tilde{\F}$ to denote the sets of all possible $f$/$\tilde{f}$.
For any $f \in \F$, its margin on sample $\x$ is defined by $M(f(\x))=f(\x)_{i^*}-\max_{i\neq i^*}f(\x)_i$, where $i^* = \arg\max_{i}f(\x)_i$. $1-M(f(\x))$ reflects the \textit{prediction uncertainty} of $f$.  Specifically, if $1-M(f(\x))$ is smaller, $f(\x)$ approaches the one-hot vector leading to low prediction uncertainty.
The expected margin of $f$ on distribution $P$ is ${M}_P(f)=\E_{(\x,y)\sim P}M(f(\x))$. The \textit{robustness} of $f$ on distribution $P$ is defined by $R_P(f) = P(\{\x,\exists\x'\in\mathcal{N}(\x),\tilde{f}(\x)\neq\tilde{f}(\x')\})$.
For any $i\in \Yc$, we denote $\tP_i = \tP(\x|y=i)$ as class-wise data distribution of the $i$-th class. $Q_i$ is similarly defined. 

\vspace{0.5\baselineskip} \noindent \textbf{Assumptions.}
To develop the theory for PDA, we assume that:\\
(i) $\forall i \in \mathcal{Y}_{\rm com}$, $\frac{Q(y=i)}{P^{c}(y=i)}\leq r$;\\
(ii) $\forall i,j \in \mathcal{Y}_{\rm com}$, the supports of $Q_i$ and $Q_j$ are disjoint for $i\neq j$;\\
(iii) For any $i\in\mathcal{Y}_{\rm com}$, $\frac{1}{2}(\tP_i+Q_i)$ satisfies $(q,\epsilon)$-constant expansion (see Definition \ref{thm:def_exp}) for some $q,\epsilon \in (0,1)$.



\vspace{0.5\baselineskip}
\begin{definition}[$(q,\epsilon)$-constant expansion]\label{thm:def_exp}
    We say that a distribution $P$ satisfies $(q,\epsilon)$-constant expansion for some $q,\epsilon \in (0,1)$, if for any set $A\subset\mathcal{X}$ with $\frac{1}{2}\geq P (A)\geq q$, we have $P(\mathcal{N}(A)\backslash A)>\min\{\epsilon,P(A)\}$.
\end{definition}

Assumption (i) is realistic because $Q(y=i)$ is finite. Assumption (ii) implies that any target sample has a unique class label. 
We follow~\cite{liu2021cycle,wei2021theoretical,pmlr-v139-cai21b} to take the expansion assumption (Assumption (iii)) of the mixture distribution. Intuitively, this assumption indicates that the conditional distributions $\tP_i$ and $Q_i$ are closely located and regularly shaped, enabling knowledge transfer from the source domain to the target domain. \citet{wei2021theoretical} justified this assumption on real-world datasets with BigGAN~\cite{brock2018large}.

\vspace{0.5\baselineskip} 
\begin{theorem}\label{thm:pda_bound}
    Suppose the above Assumptions hold. For any $f\in\F$ and any $\eta\in (0,1)$, if $f$ is $L$-Lipschiz \textit{w.r.t.} $d(\cdot,\cdot)$, we have 
    \begin{equation}
    \begin{split}
        \varepsilon_Q(f) \leq &  r\varepsilon_{\tP}(f) + c_1 R_{\frac{\tP+Q}{2}}(f)\\
        &+c_2(1-{M}_{\frac{\tP+Q}{2}}(f)) + 2rq,
    \end{split}
    \end{equation}
    where the coefficients $c_1 = \frac{2\eta r}{\min\{\epsilon,q\}(1+r)}$ and $c_2 = \frac{2r(1-\eta) }{\min\{\epsilon,q\}(1-2L\xi)(1+r)}$ are constants to $f$.
\end{theorem}

The proof is given in Appendix. Theorem~\ref{thm:pda_bound} implies that the target domain expected error $\varepsilon_Q(f)$ is bounded by the expected error $\varepsilon_{\tP}(f)$ on source domain common class data, the robustness $R_{\frac{\tP+Q}{2}}(f)$ and prediction uncertainty $1-{M}_{\frac{\tP+Q}{2}}(f)$ on mixture distribution of source and target domains.

In our method, the adversarial reweighting model aims to assign larger weights to source domain common class data, and the reweighted source data distribution is expected to approach the data distribution of source common classes. Minimizing the classification loss $\mathcal{L}_{\rm cls}$ on the reweighted source domain data distribution enforces the recognition model to predict the label of source common class data. Therefore, the expected error $\varepsilon_{\tP}(f)$ on source domain common class data could be minimized. 

We notice that $R_{\frac{\tP+Q}{2}}(f)=\frac{1}{2}(R_{\tP}(f)+R_{Q}(f))$. Since $f$ is trained using the classification loss with the class labels on the source domain, the prediction of $\tilde{f}$ on the neighborhood samples should be their class labels and thus are similar, because the neighborhood samples should belong to the same class. This implies that $R_{\tP}(f)$ is minimized. The neighborhood reciprocity clustering loss $\mathcal{L}_{\rm nrc}$ encourages similar outputs of the recognition model on neighborhood samples on target domain, which implies $R_{Q}(f)$ is minimized. 

Note that $1-{M}_{\frac{\tP+Q}{2}}(f)=\frac{1}{2}(1-{M}_{\tP}(f)) + \frac{1}{2}(1-{M}_{Q}(f))$.  By the classification loss, the outputs of $f$ on the source domain common class data are near to one-hot vectors, so that $1-{M}_{\tP}(f)$ is minimized.  Our $\alpha$-power loss $\mathcal{L}_{\rm pow}$ reduces the prediction uncertainty on the target domain, realizing the minimization of $1-{M}_{Q}(f)$ which reflects the prediction uncertainty on target domain.


\section{Experiments}\label{sec:experiment}
{We conduct experiments on benchmark datasets to evaluate our ARPM approach, and compare it with recent methods. The source code is available at \url{https://github.com/XJTU-XGU/ARPM}.}

\subsection{Setup}\label{sec:setup}
{For ease of understanding, we discuss the setup for PDA in this section. We will discuss the setup for open-set and universal DA in Sect~\ref{sec:results_os_unida} and for TTA in Sect.~\ref{sec:results_tta}. }

\vspace{0.5\baselineskip} \noindent \textbf{Datasets.}
\textit{Office-31} dataset~\cite{saenko2010adapting} contains 4,652 images of 31 categories, collected from three domains: Amazon (A), DSLR (D), and Webcam (W). 
\textit{ ImageNet-Caltech} is built with ImageNet (I)~\cite{russakovsky2015imagenet} and Caltech-256 (C)~\cite{griffin2007caltech}, respectively including 1000 and 256 classes. 
\textit{ Office-Home}~\cite{venkateswara2017deep} consists of four domains: Artistic (A), Clip Art (C), Product (P), and Real-World (R), sharing 65 classes.
\textit{ VisDA-2017}~\cite{peng2017visda} is a large-scale challenging dataset, containing two domains: Synthetic (S) and Real (R), with 12 classes.
\textit{ DomainNet}~\cite{peng2019moment} is another large-scale challenging dataset, composed of six domains with 345 classes.

 \begin{table*}[t]
	\centering
    
	\caption{Accuracy (\%) on Office-Home dataset for PDA. The best results are bolded. *AR is our conference version.}
	\setlength{\tabcolsep}{2.1pt}	
	\begin{tabular}{lcccccccccccc|c}
		\toprule
		Method                                             & A$\rightarrow$C & A$\rightarrow$P& A$\rightarrow$R & C$\rightarrow$A& C$\rightarrow$P & C$\rightarrow$R & P$\rightarrow$A & P$\rightarrow$C & P$\rightarrow$R& R$\rightarrow$A & R$\rightarrow$C & R$\rightarrow$P & Avg\\
		\midrule
		ResNet-50~\cite{He_2016_CVPR}          &46.3&67.5&75.9&59.1&59.9&62.7&58.2&41.8&74.9&67.4&48.2&74.2&61.4\\
		ADDA~\cite{tzeng2017adversarial}&45.2&68.8&79.2&64.6&60.0&68.3&57.6&38.9&77.5&70.3&45.2&78.3&62.8\\
		CDAN+E~\cite{NIPS2018_7436}&47.5&65.9&75.7&57.1&54.1&63.4&59.6&44.3&72.4&66.0&49.9&72.8&60.7\\
		IWAN~\cite{zhang2018importance} &53.9&54.5&78.1&61.3&48.0&63.3&54.2&52.0&81.3&76.5&56.8&82.9&63.6\\
		PADA~\cite{cao2018partial1} &52.0&67.0&78.7&52.2&53.8&59.0&52.6&43.2&78.8&73.7&56.6&77.1&62.1\\
		ETN~\cite{cao2019learning} &59.2&77.0&79.5&62.9&65.7&75.0&68.3&55.4&84.4&75.7&57.7&84.5&70.5\\
		DRCN~\cite{li2020deep}&54.0&76.4&83.0&62.1&64.5&71.0&70.8&49.8&80.5&77.5&59.1&79.9&69.0 \\
		BA$^3$US~\cite{liang2020balanced}&60.6&83.2&88.4&71.8&72.8&83.4&75.5&61.6&86.5&79.3&62.8&86.1&76.0\\
        ISRA+BA$^3$US~\cite{xiao2021implicit}&{64.7}&83.0&89.1&75.7&75.5&85.4&78.5&{64.2}&88.1&81.3&65.3&86.7&78.2\\ 
        SHOT++~\cite{9512429}&{65.0} &85.8 &\bf93.4& \bf 78.8& 77.4 & {87.3} &79.3 &{66.0} &89.6 &81.3 &{68.1} &86.8 &79.9 \\
        SPDA~\cite{guo2022selective} &64.2&\bf 87.8& 88.0& 74.3&75.1& 79.1& 79.4& 58.9& 85.1& 81.4& 67.4& 84.1& 77.1\\
        APDA-CI~\cite{lin2022adversarial}&61.7&  {86.9}& 90.5& 77.2& 76.9& 83.8& 79.6& 63.8& 88.5&  {85.0}& 65.8& 86.2& 78.8\\
        CLA~\cite{9705553}&66.7&85.6&90.9&75.6&76.9&86.8&78.8& {67.4}&88.7&81.7&66.9&87.8&79.5\\ 
        RAN~\cite{9957101}&63.3&83.1&89.0& 75.0&74.5&82.9&78.0&61.2&86.7&79.9&63.5&85.0&76.8\\
        STCPDA~\cite{he2023addressing}&63.1&\bf87.8&90.1&77.2&75.4& 85.6&\bf 81.4&62.4&\bf 90.5&82.6& {69.5}&88.2&79.5\\
        SLM~\cite{sahoo2021select}&61.1&84.0&91.4&76.5&75.0&81.8&74.6&55.6&87.8&82.3&57.8&83.5&76.0\\
        SAN++~\cite{9736609}&61.3& 81.6& 88.6& 72.8& 76.4& 81.9& 74.5& 57.7&87.2&79.7& 63.8& 86.1& 76.0\\
        IDSP~\cite{9983498}&60.8&80.8&87.3&69.3&76.0&80.2&74.7&59.2&85.3&77.8&61.3&85.7&74.9\\
    MOT~\cite{Luo_2023_CVPR}&63.1&{86.1}& {92.3}& {78.7}&\bf85.4&\bf89.6&{79.8}&62.3& {89.7} &{83.8}&{67.0}&\bf89.6& {80.6}\\
        \midrule
        $^*$AR~\cite{gu2021adversarial}& {67.4}&85.3&90.0&77.3&70.6&85.2&79.0&64.8&89.5&80.4&66.2&86.4&78.3\\
        \bf ARPM&\bf 68.3&\bf87.8& {92.3}&{77.8}& {84.6}&{86.3}& {81.1}&\bf69.2&{89.5}&\bf86.2&\bf70.0& {89.1}&\bf81.8 \\
       \bottomrule
	\end{tabular}
	\label{tab:result_office-home}
 \end{table*}

\begin{table}[t]
	\centering
 
	\caption{Accuracy (\%) on ImageNet-Caltech dataset for PDA. The best results are bolded. *AR is our conference version.}
	\setlength{\tabcolsep}{9pt}
   \normalsize
		\begin{tabular}{lcc|c}
			\toprule
			Method&C$\rightarrow$I&I$\rightarrow$C&Avg\\ 
			\midrule
			ResNet-50~\cite{He_2016_CVPR} &     71.3&69.7&70.5 \\
			DAN~\cite{pmlr-v37-long15}&60.1&71.3&65.7\\
			DANN~\cite{ganin2016domain}&67.7&70.8&69.2\\
			IWAN~\cite{zhang2018importance}&  73.3&78.1&75.7\\
			PADA~\cite{cao2018partial1}  &          70.5&75.0&72.8\\
			ETN~\cite{cao2019learning}   &          74.9&83.2&79.1\\
			DRCN~\cite{li2020deep}        &           78.9&75.3&77.1\\
			BA$^3$US~\cite{liang2020balanced}& 83.4&84.0&83.7\\
            ISRA+BA$^3$US~\cite{xiao2021implicit}& {83.7}&85.3& {84.5}\\
            SLM~\cite{sahoo2021select}&81.4&82.3&81.9\\
             SAN++~\cite{9736609}&81.1&83.3&82.2\\ 
            \midrule
            $*$AR~\cite{gu2021adversarial}&82.2&\bf 87.1&84.7\\
            \bf ARPM&\bf 87.1&  {84.6}&\bf 85.9\\
            \bottomrule
	\end{tabular}  
	\label{tab:results_imagenet}
\end{table}

\begin{table}[t]
	\centering
 
	\caption{Accuracy (\%) on VisDA-2017 dataset for PDA. The best results are bolded. *AR is our conference version.}
	\setlength{\tabcolsep}{30pt}
   \normalsize
		\begin{tabular}{lc}
			\toprule
			Method&S$\rightarrow$R\\ 
			\midrule
			ResNet-50~\cite{He_2016_CVPR} & 45.3 \\
			IWAN~\cite{zhang2018importance}&48.6\\
			PADA~\cite{cao2018partial1}  &53.5 \\
			DRCN~\cite{li2020deep}        &58.2\\
            SHOT++~\cite{9512429}&78.6 \\
            SPDA~\cite{guo2022selective} &82.9\\
            APDA-CI~\cite{lin2022adversarial}&69.8\\
            RAN~\cite{9957101}&75.1\\
            STCPDA~\cite{he2023addressing}&70.1\\
            SLM~\cite{sahoo2021select}& 91.7\\
            SAN++~\cite{9736609}&63.1\\
            MOT~\cite{Luo_2023_CVPR}& {92.4}\\
            \midrule
            $*$AR~\cite{gu2021adversarial}&88.8\\
            \bf ARPM& \bf 93.2\\
            \bottomrule
	\end{tabular}  
	\label{tab:results_visda}
\end{table}

\vspace{0.5\baselineskip} \noindent \textbf{Adaptation tasks.} On Office-31, ImageNet-Caltech, Office-Home datasets, we set every domain as the source domain in turn and use each of the rest domain(s) to build the target domain, forming the adaptation tasks. On Office-31, we follow~\cite{cao2018partial} to select images from the 10 categories shared by Office-31 and Caltech-256~\cite{griffin2007caltech} to build the target domain in each task. On ImageNet-Caltech dataset, we utilize the 84 shared classes by ImageNet and Caltech-256 to build the target domain in each task. As most networks are pre-trained on the training set of ImageNet,  for task C$\rightarrow$I, we use images from ImageNet validation set to build the target domain. On Office-Home, we use images of the first 25 classes in alphabetical order to build the target domain in each task. On VisDA-2017, we set Synthetic (S) as source domain and Real (R) as target domain to perform synthetic to real domain transfer, and use the first 6 classes in alphabetical order as the target domain. 
On DomainNet, since the labels of some domains and classes are very noisy, we follow~\cite{saito2019semi} to adopt four domains (Clipart (C), Painting (P), Real (R), and Sketch (S)) with 126 classes for PDA. We use the first 40 classes in alphabetical order to build the target domain in each task. 
In each adaptation task, we report the average classification accuracy on target domain.

\vspace{0.5\baselineskip} \noindent \textbf{Implementation details.}
We implement our method using Pytorch~\cite{NEURIPS2019_bdbca288} on a single Nvidia Tesla v100 GPU.
We use the SGD algorithm with momentum 0.9 to update the parameters of $F$ and $C$.  The learning rate of $C$ is ten times that of $F$. The parameters of $D$ are updated by the Adam algorithm with learning rate 0.001.  Following~\cite{pmlr-v37-ganin15}, we adjust the learning rate of $C$ by $ \frac{\kappa}{(1+10p)^{0.75}}$, where $p$ represents the training progress linearly changing from 0 to 1. Bath size is set to 64. For Office-Home, ImageNet-Caltech, and DomainNet datasets, we set $\kappa=0.01$ and $\lambda=0.3$. Since the training processes on VisDA-2017 and Office-31 datasets converge faster, we set $\kappa$ to 0.001 for VisDA and 0.005 for Office-31. Correspondingly, $\lambda$ is set to 1.0 for VisDA-2017 and Office-31 datasets. $N$ in Algorithm~\ref{alg:training} is set to 500 for Office-Home and Office-31 datasets, and is set to 1000 for the larger VisDA-2017, DomainNet, and ImageNet-Caltech datasets. $N'$ is set to $64*2000$.

On VisDA-2017, ImageNet-Caltech, and DomainNet datasets, we sample the mini-batch data according to the learned weights using a reweighted random sampler and then calculate the unweighted classification loss on the sampled mini-batch data in training. We find that this strategy makes training more stable on these datasets than the commonly used strategy that we first uniformly sample mini-batch data and then reweight the classification loss for each sample in the mini-batch. The reason could be that the number of samples with zero weights is large in these large-sized datasets and the uniformly sampled mini-batch data may contain only a few samples having non-zero weights.

 \begin{table*}[t]
	\centering
 
	\caption{Accuracy (\%) on Office-31 dataset for PDA.  The best results are bolded. *AR is our conference version.}
	\setlength{\tabcolsep}{11.6pt}
		\begin{tabular}{lcccccc|c}
			\toprule
			Method&A$\rightarrow$D&A$\rightarrow$W & D$\rightarrow$A &D$\rightarrow$W&W$\rightarrow$A& W$\rightarrow$D&Avg\\ 
			\midrule
			ResNet-50~\cite{He_2016_CVPR}      &83.4&75.6&83.9&96.3&85.0&98.1&87.1  \\
			DAN~\cite{pmlr-v37-long15}&61.8&59.3&75.0&73.9&67.6&90.5&71.3\\
			DANN~\cite{ganin2016domain}&81.5&73.6&82.8&96.3&86.1&98.7&86.5\\
			IWAN~\cite{zhang2018importance}  &90.5&89.2&95.6&99.3&94.3&99.4&94.7\\
			PADA~\cite{cao2018partial1}             &82.2&86.5&92.7&99.3&95.4&\bf100.0&92.7\\
			ETN~\cite{cao2019learning}              &95.0&94.5&96.2&\bf100.0&94.6&\bf100.0&96.7\\
			DRCN~\cite{li2020deep}                    &86.0&88.5&95.6&\bf100.0&95.8&\bf100.0&94.3\\
			TSCDA~\cite{ren2020learning} &98.1&96.8&  94.8&\bf 100.0& 96.0& \bf100.0& 97.6\\
			BA$^3$US~\cite{liang2020balanced} & {99.4}&99.0&94.8&\bf100.0&95.0&98.7&97.8 \\
            ISRA+BA$^3$US~\cite{xiao2021implicit}&98.7&99.3&95.4&\bf100.0&95.4&\bf100.0&98.2\\
            SPDA~\cite{guo2022selective} & 96.2&99.3& 96.0& \bf100.0&96.6&  \bf100.0& 98.0\\
            APDA-CI~\cite{lin2022adversarial}&96.8& 99.7& 96.2& \bf100.0&96.6& \bf100.0&98.2\\
            CLA~\cite{9705553}&\bf100.0&\bf100.0&94.5&\bf100.0&96.7&\bf100.0&98.5\\ 
            RAN~\cite{9957101}&97.8&99.0& 96.3&\bf100.0&96.2& \bf100.0&98.2\\
            SLM~\cite{sahoo2021select}& 98.7& 99.8& 96.1&\bf100.0& 95.9& {99.8}&  {98.4}\\
            SAN++~\cite{9736609}&98.1&99.7&94.1&\bf100.0&95.5&\bf100.0&97.9\\
            IDSP~\cite{9983498}&99.4&99.7&95.1&99.7&95.7&\bf100.0&98.3\\
            MOT~\cite{Luo_2023_CVPR}&98.7&99.3&96.1&\bf 100.0& {96.4}&\bf 100.0& {98.4}\\
            \midrule
            $^*$AR~\cite{gu2021adversarial}&96.8&93.5&95.5&\bf 100.0&96.0&99.7&96.9\\
            \bf ARPM&99.6& {99.4}&\bf96.6& {99.9}&\bf 96.8&\bf 100.0&\bf 98.7 \\
            \bottomrule
	\end{tabular}  
	\label{tab:results_office}
\end{table*}

\begin{table*}[t]
	\centering
 
	\caption{Accuracy ($\%$) on DomainNet dataset for PDA.   The best results are bolded. *AR is our conference version. The results of the compared methods are produced using their official source codes.}
	\setlength{\tabcolsep}{2.6pt}
	\begin{tabular}{lcccccccccccc|c}
		\toprule
		Method    &C$\rightarrow$P& C$\rightarrow$R& C$\rightarrow$S& P$\rightarrow$C& P$\rightarrow$R& P$\rightarrow$S& R$\rightarrow$C& R$\rightarrow$P& R$\rightarrow$S& S$\rightarrow$C& S$\rightarrow$P& S$\rightarrow$R& Avg\\
		\midrule
		ResNet-50~\cite{He_2016_CVPR}          &41.2&60.0&42.1&54.5&70.8&48.3&63.1&58.6&50.3&45.4&39.3&49.8&52.0\\
		DANN~\cite{ganin2016domain}&27.8&36.6&29.9&31.8&42.0&36.6&47.6&46.8&40.9&25.8&29.5&32.7&35.7\\
		CDAN+E~\cite{NIPS2018_7436}&37.5&48.3&46.6&45.5&61.0&52.6&62.0&60.6&54.7&35.4&38.5&43.6&48.9\\
		PADA~\cite{cao2018partial1} &22.5&32.9&30.0&25.7&56.5&30.5&65.3&63.4&54.2&17.5&23.9&26.9&37.4\\
		BA$^3$US~\cite{liang2020balanced}&42.9&54.7&53.8& {64.0}& {76.4}&64.7& {80.0}&74.3&74.0& {50.4}&42.7&49.7&60.6\\
        ISRA+BA$^3$US~\cite{xiao2021implicit}& {43.3}& {55.1}& {56.9}&59.4&75.8& {66.3}&76.3& {76.1}& {77.3}&44.2& {50.6}& {50.6}& {61.0}\\
        STCPDA~\cite{he2023addressing}& 65.1& 69.6&\bf 69.6& 72.7& 77.6&78.7& 78.1&72.9&\bf 80.0&64.4&60.7&67.8&71.4\\
        \midrule
        $^*$AR~\cite{gu2021adversarial}&52.7& 68.2& 58.3& 66.8& 77.5& 74.4& 76.7& 71.8& 70.5& 53.7& 53.6& 61.6& 65.5\\
        \bf ARPM&\bf 67.9&\bf 79.8& 66.3&\bf 78.4&\bf 84.1&\bf 81.9&\bf 86.5&\bf 78.0&78.6&\bf 62.5&\bf 64.8&\bf 71.7&\bf 75.0 \\
        \bottomrule
	\end{tabular}
	\label{tab:result_domainnet}
\end{table*}


{
\subsection{Results for PDA}\label{sec:result}

We implement our approach with three different random seeds $\{2019,2021,2023\}$, and report the average results over these three different runs in Tables~\ref{tab:result_office-home},~\ref{tab:results_office},~\ref{tab:result_domainnet},~\ref{tab:results_imagenet}, and~\ref{tab:results_visda} for Office-Home, Office-31, DomainNet, ImageNet-Caltech, and VisDA-2017 datasets, respectively. The results of compared methods on Office-Home, Office-31, ImageNet-Caltech, and VisDA-2017 datasets are quoted from their papers. The results of compared methods  on DomainNet dataset are produced using their official codes (except the results of STCPDA~\cite{he2023addressing}, which are quoted from its paper). 

\begin{table*}
\caption{H-score (\%)  on Office-Home of different open-set DA methods. The best results are bolded.
}
\label{tab:results_open-set DA}
\centering
\setlength{\tabcolsep}{3.0pt}
\begin{tabular}{l|cccccccccccc|c}
	\toprule                                                                                                            Method&A$\rightarrow$C&A$\rightarrow$P&A$\rightarrow$R&C$\rightarrow$A&C$\rightarrow$P&C$\rightarrow$R&P$\rightarrow$A&P$\rightarrow$C&P$\rightarrow$R&R$\rightarrow$A&R$\rightarrow$C&R$\rightarrow$P&Avg\\
    \midrule
	STA \cite{liu2019separate} & 55.8 & 54.0 & 68.3 & 57.4 & 60.4 & 66.8 & 61.9 & 53.2 & 69.5 & 67.1 & 54.5 & 64.5 & 61.1 \\
	OSBP \cite{saito2018open} & 55.1 & 65.2 & 72.9 & 64.3 & 64.7 & 70.6 & 63.2 & 53.2 & 73.9 & 66.7 & 54.5 & 72.3&64.7 \\
	ROS \cite{bucci2020effectiveness} & 60.1 & 69.3 & 76.5 & 58.9 & 65.2 & 68.6 & 60.6 & 56.3 & 74.4 & 68.8 & 60.4 & 75.7&66.2\\
	DCC \cite{li2021domain} & 56.1 & 67.5 & 66.7 & 49.6 & 66.5 & 64.0 & 55.8 & 53.0 & 70.5 & 61.6 & 57.2 & 71.9 &  61.7 \\
	OVANet \cite{saito2021ovanet} &  58.6 & 66.3 & 69.9 & 62.0 & 65.2 & 68.6 & 59.8 & 53.4 & 69.3 & 68.7 & 59.6 & 66.7 &  64.0 \\
    UMAD \cite{liang2021umad} & 59.2 & 71.8 & 76.6 & 63.5 & 69.0 & 71.9 & 62.5 & 54.6 & 72.8 & 66.5 & 57.9 & 70.7 &  66.4 \\
	GATE \cite{chen2022geometric} &  63.8 & 70.5 & 75.8 & 66.4 & 67.9 & 71.7 & 67.3 & 61.5 & 76.0 & 70.4 & 61.8 & 75.1 &  69.1 \\	
    GLC \cite{qu2023upcycling} & \textbf{65.3} & 74.2 & 79.0 & 60.4 &71.6 & \bf74.7 & 63.7 & \bf63.2 &75.8 & 67.1 &64.3 & 77.8 &  69.8 \\
    PPOT~\cite{Yang_Gu_Sun_2023}& 60.7 & 75.2 & 79.5 & \bf 67.3 & 70.1 & 73.8 & \bf 70.6 & 57.2 & 76.1 &71.8 & 61.4 & 75.8 &   70.0 \\
    ANNA~\cite{li2023adjustment}&69.9 &73.7 &76.8 &64.7 &68.6 &73.0 &66.5 &63.1 &76.6 &71.6 &\bf65.7 &78.7 &70.7\\
    \midrule
    \bf ARPM&63.8&\bf76.0&\bf80.6 &\bf67.3& \bf 72.1&74.6&67.7&61.7& \bf76.7& \bf 73.8& 65.4&\bf81.2&\bf71.7\\
	\bottomrule
\end{tabular}
\end{table*}

In Table~\ref{tab:result_office-home}, our approach ARPM achieves the best average accuracy of 81.8\% on Office-Home dataset, outperforming the second-best approach of MOT by 1.2\%. 
Tables~\ref{tab:results_imagenet} and~\ref{tab:results_visda} imply that our proposed ARPM achieves the best results of 85.9\% and 93.2\% on ImageNet-Caltech and ViSDA-2017, respectively. ARPM outperforms the second-best method on ImageNet-Caltech by 1.2\% and on VisDA-2017 by 0.8\%, respectively. 
In Table~\ref{tab:results_office}, on Office-31 dataset, our proposed ARPM achieves the best result of 98.7\%, outperforming the recent PDA methods CLA~\cite{9705553} (the second-best method) by 0.2\%. ARPM outperforms CLA by 2.3\% on Office-Home dataset. 
Table~\ref{tab:result_domainnet} shows that our approach ARPM achieves the best accuracy on DomainNet dataset, outperforming the second-best method by 3.6\%. 

Among the reported methods in Tables \ref{tab:result_office-home}, \ref{tab:results_imagenet}, \ref{tab:results_visda}, \ref{tab:results_office}, and \ref{tab:result_domainnet},  DANN~\cite{pmlr-v37-ganin15}, DAN~\cite{pmlr-v37-long15}, ADDA~\cite{tzeng2017adversarial}, and CDAN+E~\cite{NIPS2018_7436} are devised for closed-set DA which do not consider the challenge of label space mismatch in PDA and achieve worse results than the other PDA approaches. In PDA approaches, IWAN~\cite{zhang2018importance}, PADA~\cite{cao2018partial1}, ETN~\cite{cao2019learning}, DRCN~\cite{li2020deep}, TSCDA~\cite{ren2020learning}, and BA$^3$US~\cite{liang2020balanced} align reweighted source domain feature distribution and unweighted target domain feature distribution by minimizing maximum mean discrepancy or adversarial training, performing worse than MOT~\cite{Luo_2023_CVPR}. MOT~\cite{Luo_2023_CVPR} aligns reweighted source  and unweighted target feature distributions using robust optimal transport that does not require exact matching of the distributions. This implies that aligning reweighted distribution without requiring exact matching may be better than that with requiring exact matching for PDA. Our proposed ARPM boosts the performance of MOT on all the evaluated datasets. ARPM adversarially learns to reweight source domain data by minimizing the cross-domain distribution distance, and enforcing robustness and reducing prediction uncertainty of the recognition model,  which more effectively tackles PDA from a novel perspective with theoretical analysis. 

 Compared with the conference version, the results on Office, Office-Home, ImageNet-Caltech, VisDA-2017, and DomainNet datasets are improved by 1.8\%, 3.5\%, 1.2\%, 4.4\%, and 9.5\% in this journal version, respectively.  In this journal version, we extend the work by introducing more methodological techniques, \eg, the $\alpha$-power maximization for substituting entropy minimization, the neighborhood reciprocity clustering~\cite{yang2021exploiting}, and the spectral normalization to enforce the Lipschitz constraint and the PCA-based initialization strategy to improve stability. The performance improvements demonstrate the effectiveness of the methodological extensions for PDA. Note that the usefulness of each of these techniques will be verified in Sect.~\ref{sec:analysis}.
}

\begin{table*}
\caption{H-score (\%)  on Office-Home of different universal DA methods. The best results are bolded.
}
\label{tab:results_unida}
\centering
\setlength{\tabcolsep}{3.0pt}
\begin{tabular}{l|cccccccccccc|c}
	\toprule                                                                                                            Method&A$\rightarrow$C&A$\rightarrow$P&A$\rightarrow$R&C$\rightarrow$A&C$\rightarrow$P&C$\rightarrow$R&P$\rightarrow$A&P$\rightarrow$C&P$\rightarrow$R&R$\rightarrow$A&R$\rightarrow$C&R$\rightarrow$P&Avg\\
    \midrule
	UAN \cite{you2019universal} & 51.6 & 51.7 & 54.3 & 61.7 & 57.6 & 61.9 & 50.4 & 47.6 & 61.5 & 62.9 & 52.6 & 65.2 & 56.6 \\
    CMU \cite{fu2020learning} &56.0 & 56.9 & 59.2 & 67.0 & 64.3 & 67.8 & 54.7 & 51.1 & 66.4 & 68.2 & 57.9 & 69.7 & 61.6 \\
    DANCE \cite{saito2020universal} &61.0 & 60.4 & 64.9 & 65.7 & 58.8 & 61.8 & 73.1 & 61.2 & 66.6 & 67.7 & 62.4 & 63.7 & 63.9 \\
    USFDA \cite{kundu2020universal}  & 60.3 & 79.7 & 80.7 & 64.2 & 67.9 & 79.3 & 74.1 & 65.3 & 80.1 & 68.6 & 59.8 & 77.9 & 71.4 \\ 
    UMAD \cite{liang2021umad} &  61.1 & 76.3 & 82.7 & 70.7 & 67.7 & 75.7 & 64.4 & 55.7 & 76.3 & 73.2 & 60.4 & 77.2 & 70.1  \\
    DCC \cite{li2021domain} & 58.0 & 54.1 & 58.0 & 74.6 & 70.6 & 77.5 & 64.3 & 73.6 & 74.9 & 81.0 & \bf 75.1 & 80.4 & 70.2 \\
    OVANet \cite{saito2021ovanet} &  63.4 & 77.8 & 79.7 & 69.5 & 70.6 & 76.4 & 73.5 & 61.4 & 80.6 & 76.5 & 64.3 & 78.9 & 72.7 \\
    GATE \cite{chen2022geometric} & 63.8 & 75.9 & 81.4 & 74.0 & 72.1 & 79.8 & 74.7 & \bf 70.3 & 82.7 & 79.1 & 71.5 & 81.7 & 75.6 \\
    UniOT \cite{chang2022unified} & \bf 67.3 & 80.5 & 86.0 & 73.5 & 77.3 & 84.3 & 75.5 & 63.3 & 86.0 & 77.8 & 65.4 & 81.9 & 76.6\\
    GLC \cite{qu2023upcycling}&  64.3 & 78.2 & \bf 89.6 & 63.1 & \bf81.7 & \bf 89.1 & 77.6 & 54.2 & \bf 88.9 & 80.7 & 54.2 & \bf 85.9 & 75.6 \\
    PPOT~\cite{Yang_Gu_Sun_2023} & 66.0 & 79.3 & 84.8 & \bf 78.8 &  78.0 & 80.4 & \bf 82.0 & 62.0 & 86.0 & \bf 82.3 & 65.0 & 80.8 & 77.1\\
    SAKA~\cite{wang2023exploiting}&64.3&80.4&86.1&72.0&71.1&77.8&71.5&61.7&83.8&79.1&64.8&82.4&74.6\\
    \midrule
    \bf ARPM&65.2&\bf 81.2&89.4&73.2&73.4&83.9&74.9&67.3&84.8&78.9&70.3 &85.1 &\bf 77.3 \\
	\bottomrule
\end{tabular}
\end{table*}

\begin{table*}
    \centering
    
    \caption{ Accuracy (\%) on ImageNet-R of different TTA methods. The best result is bolded.}
    \setlength{\tabcolsep}{4.5pt}
   \normalsize
        \begin{tabular}{l|cccccc}
            \toprule
            
            Method&Source-trained model&TTT~\cite{sun2020test} &NORM~\cite{schneider2020improving} &TENT~\cite{wang2021tent} &DUA~\cite{mirza2022norm}&\bf TPM (ours)\\ 
            \midrule
            Accuracy&33.0&33.5&34.7&36.8&33.2&\bf 38.3\\
            \bottomrule
    \end{tabular}  
    \label{tab:results_tta}
\end{table*}

 \begin{table}[t]
	\centering
	\caption{Ablation study on VisDA-2017 dataset. ``SO'' means training the model using the source domain data only. ``R'' represents our adversarial reweighting model. ``P'' is the $\alpha$-power loss. ``N'' is the neighborhood reciprocity clustering loss. ``SO+R+N+P'' is our full method ARPM. We also report the results for entropy minimization (E) as a substitute of $\alpha$-power loss (P).}
	\setlength{\tabcolsep}{21pt}
    \normalsize
	\begin{tabular}{lc}
		\toprule
		Method & S$\rightarrow$R \\
		\midrule
		SO& 46.4\\
        SO+R & 50.4\\
        SO+P &90.7\\
        SO+N & 89.2\\
        SO+R+P & 91.7\\
        SO+R+N & 90.8\\
        SO+N+P & 92.4\\
        SO+R+N+P (ARPM)&\bf 93.2 \\
        \midrule
        \midrule
        SO+E &85.2 \\
        SO+R+E &89.2\\
        SO+E+R+N &90.7\\
       \bottomrule
	\end{tabular}
	\label{tab:ablation_visda}
 \end{table}

 \begin{table*}[t]
	\centering
	\caption{Ablation study on Office-Home dataset. ``SO'' means training the model using the source domain data only. ``R'' represents our adversarial reweighting model. ``P'' is the $\alpha$-power loss. ``N'' is the neighborhood reciprocity clustering loss. ``SO+R+N+P'' is our full method ARPM. We also report the results for entropy minimization (E) as a substitute of $\alpha$-power loss (P).}
	\setlength{\tabcolsep}{1.5pt}	
	\begin{tabular}{lcccccccccccc|c}
		\toprule
		Method & A$\rightarrow$C & A$\rightarrow$P& A$\rightarrow$R & C$\rightarrow$A& C$\rightarrow$P & C$\rightarrow$R & P$\rightarrow$A & P$\rightarrow$C & P$\rightarrow$R& R$\rightarrow$A & R$\rightarrow$C & R$\rightarrow$P & Avg\\
		\midrule
		SO&50.2&69.5&79.8&60.2&61.1&67.5&60.3&43.7&74.5&71.3&50.6&78.0&63.9\\
        SO+R&50.3&71.4&83.3&61.9&64.3&72.6&63.8&43.9&78.5&73.3&52.6&80.5&66.4\\
        SO+P&64.3&84.4&88.8&76.6&80.0&81.9&78.2&65.4&88.4&81.0&63.1&86.3&78.2 \\
        SO+N&63.1&83.8&88.3&75.1&76.6&77.4&77.4&60.8&85.1&81.3&62.9&86.1&76.5 \\
        SO+R+P&67.5&87.6&91.7&79.4&78.9&85.0&78.9&65.3&\bf90.8&82.8&63.7&87.7&79.9 \\
        SO+R+N&65.3&86.1&92.0&\bf80.3&77.2&85.4&\bf81.2&64.2&87.8&84.5&65.3&89.0&79.9\\
        SO+N+P&65.6&86.5&91.3&78.6&82.1&82.3&79.8&65.6&90.7&82.2&65.0&86.9&79.7\\
        SO+R+N+P (ARPM)&\bf 68.3&\bf87.8&\bf{92.3}&{77.8}&\bf{84.6}&\bf{86.3}&81.1&\bf69.2&{89.5}&\bf86.2&\bf70.0&\bf{89.1}&\bf81.8 \\
        \midrule
        \midrule
        SO+E & 58.6 &84.1 &88.0  &74.1 &76.8 &77.4 &76.1 &61.4 &86.2 &78.5 &61.5  &84.3  &75.6 \\
        SO+R+E &63.3  &86.7 &91.5  &78.0  &77.6  &81.9 &77.5  &64.6  &89.7 &81.0  &63.6  &86.8 &78.4 \\
        SO+E+R+N &66.6&86.8&91.9&78.3&81.9&85.3&77.9&65.9&89.6&83.7&67.9&89.6&80.5 \\
       \bottomrule
	\end{tabular}
	\label{tab:ablation_office-home}
 \end{table*}

{
\subsection{Results for Open-set and Universal DA}\label{sec:results_os_unida}
We follow the protocol of~\cite{Yang_Gu_Sun_2023} to conduct experiments on Office-Home dataset for open-set and universal DA. In the test phase, the samples with prediction confidence lower than 0.65 are classified as ``unknown'' class. $\tau$ is set to 0.25 and $\lambda'$ is set to 0.05.
The other experimental details are the same as those for PDA. We report the open-set evaluation metric, H-score, in Tables~\ref{tab:results_open-set DA} and~\ref{tab:results_unida}. 
It can be observed from Tables~\ref{tab:results_open-set DA} and~\ref{tab:results_unida} that our method performs better than the recent open-set DA methods and universal DA methods on Office-Home dataset. Our proposed ARPM outperforms the previous state-of-the-art method PPOT~\cite{Yang_Gu_Sun_2023} by 1.7\% in open-set DA task as in Table~\ref{tab:results_open-set DA} and by 0.2\% in universal DA task. These results imply that our method can tackle the tasks of open-set DA and universal DA with open-class data.

\subsection{Results for TTA}\label{sec:results_tta}
For TTA, we take the RestNet-18 pre-trained on ImageNet as the source-trained model and evaluate it on ImageNet-R dataset~\cite{hendrycks2021many}. The experimental setups are the same as those in~\cite{wang2021tent}. The experimental results are reported in Table~\ref{tab:results_tta}. We can see that our proposed method TPM (test $\alpha$-power maximization) achieves better results than the other TTA approaches. Especially, TPM outperforms TENT. The model adaptation is performed by $\alpha$-power maximization in TPM and by entropy minimization in TENT.
The performance improvement of TPM over TENT indicates that our $\alpha$-power maximization could be more effective than entropy minimization for TTA.

}

 \subsection{Analysis}\label{sec:analysis}
 In this section, for convenience of description, we utilize ``SO'' to denote the baseline approach that trains the recognition model using the source domain data only. ``R'' represents our adversarial reweighting model. ``P'' is the $\alpha$-power loss. ``N'' is the neighborhood reciprocity clustering loss.
 
 \vspace{0.5\baselineskip} \noindent \textbf{Effectiveness of components in ARPM.}
 We study the effectiveness of each component in our method on VisDA-2017 and Office-Home datasets, of which the results are reported in Tables~\ref{tab:ablation_visda} and~\ref{tab:ablation_office-home}. Note that the differences between SO and ``ResNet-50'' in Tables~\ref{tab:result_office-home} and~\ref{tab:ablation_visda} are that SO utilizes label smoothing and feature normalization. These two techniques enable SO to perform slightly better than ResNet-50. We can observe from Tables~\ref{tab:ablation_visda} and~\ref{tab:ablation_office-home} that on both datasets, SO+R, SO+R+P, SO+R+N, and SO+R+N+P (\ie, our full method ARPM) respectively outperforms SO, SO+P, SO+N, and SP+N+P, demonstrating the effectiveness of our adversarial reweighting model for learning to reweight source domain data. The results of SO+P, SO+R+P, SO+N+P, SO+R+N+P are better than those of SO, SO+R, SO+N, SO+R+N, respectively. This confirms the usefulness of our $\alpha$-power maximization for reducing the prediction uncertainty of the recognition model. Tables~\ref{tab:ablation_visda} and~\ref{tab:ablation_office-home} show that SO+N+P outperforms both SO+N and SO+P, which implies that $\alpha$-power maximization and neighborhood reciprocity clustering are complementary to improve the performance. On all the combinations of SO, P, N, and R, SO+R+P+N, \ie, ARPM, achieves the best results on both datasets. Note that SO and SO+R do not utilize target domain data to update the recognition model (SO+R utilizes target domain data to learn the source data weight). The results of SO and SO+R seem to be largely lower than those of the other approaches in Tables~\ref{tab:ablation_visda} and~\ref{tab:ablation_office-home} since the $\alpha$-power and neighborhood reciprocity clustering losses are implemented on target domain data.

 \begin{figure}[t]
	\centering
	\subfigure[] { \includegraphics[width=0.44\columnwidth]{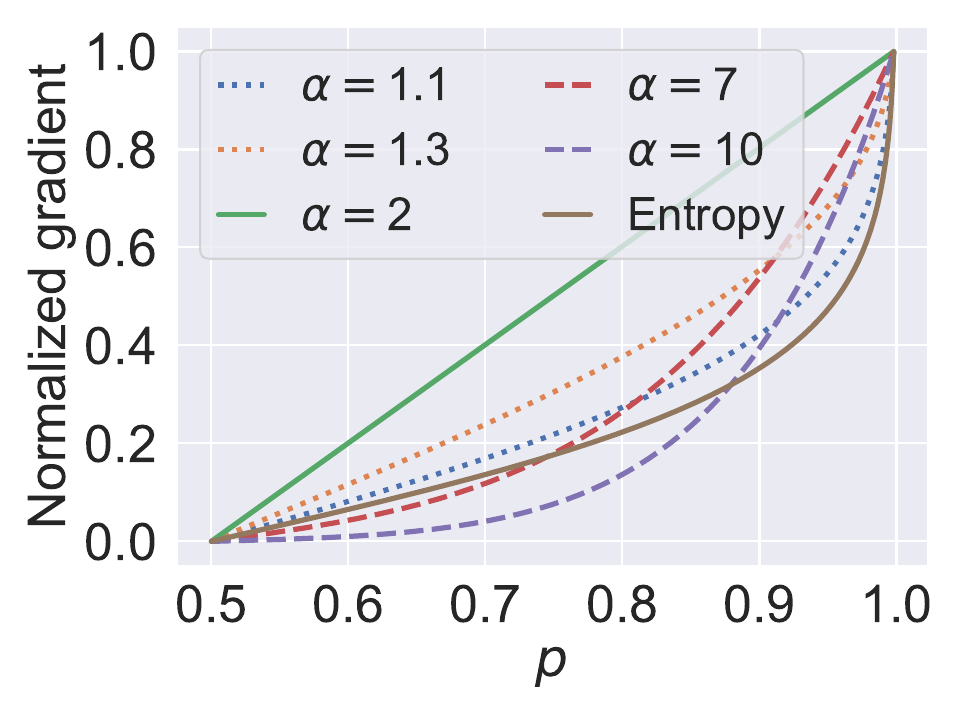}
		\label{fig:grad} }
	\subfigure[]{ \includegraphics[width=0.44\columnwidth]{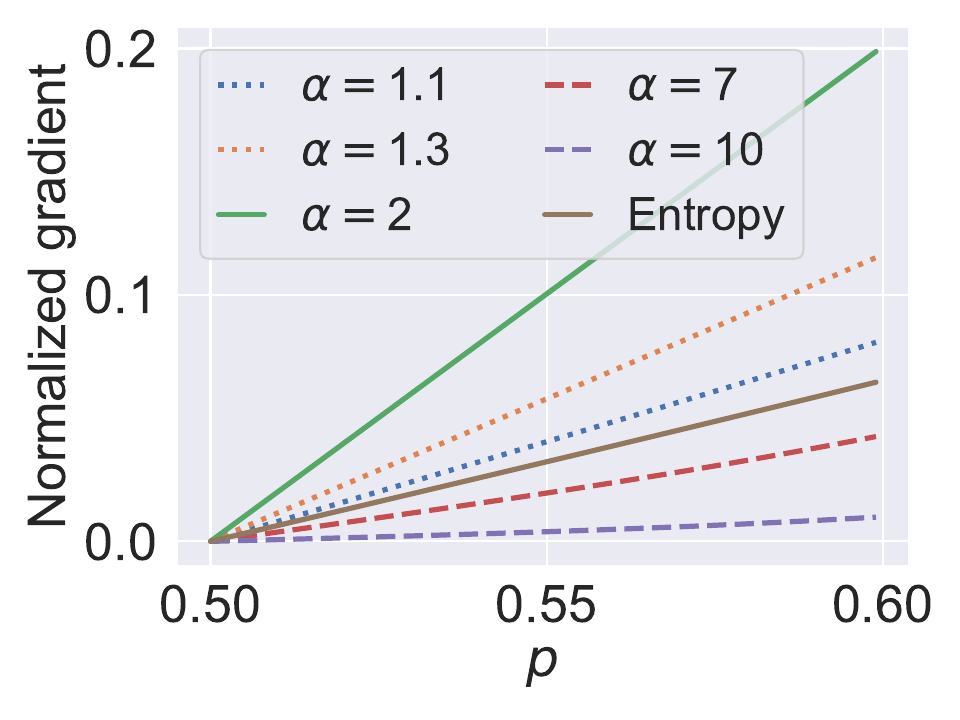} 
		\label{fig:grad_small}}
	
	\caption{(a) The normalized gradient of $\alpha$-power loss with varying $\alpha$ and entropy loss. (b) Amplified part of Fig.~\ref{fig:grad} with $p$ ranging in $[0.5,0.6]$.
	}
	\label{fig:comparison_alpha_entropy}
\end{figure}
 \vspace{0.5\baselineskip} \noindent \textbf{Comparison of $\alpha$-power maximization and entropy minimization.}
 We report the results of entropy minimization (E) as a substitute of $\alpha$-power maximization (P) in our framework. Tables~\ref{tab:ablation_visda} and~\ref{tab:ablation_office-home} show that SO+E, SO+R+E, and SO+R+E+N respectively degrade the results of SO+P, SO+R+P, and SO+R+P+N by more than 1.0\% on both VisDA-2017 and Office-Home datasets. This demonstrates that our $\alpha$-power maximization is more effective than entropy minimization for PDA. 
 
 We provide more analysis on the $\alpha$-power maximization and entropy minimization in this paragraph. We take the two-way classification task as an example to compare the gradients of the $\alpha$-power loss (\ie, $\mathcal{H}(p)=p^{\alpha}+(1-p)^{\alpha}$) and entropy loss (\ie, $\mathcal{H}(p)=p\log p+(1-p)\log(1-p)$) \textit{w.r.t.} the probability $p$ in Fig.~\ref{fig:comparison_alpha_entropy}. Note we only care about the absolute value of the gradients in this experiment. The gradients are normalized by dividing by their maximum value when $p$ ranges in $[0.5,0.99]$. The samples with larger normalized gradients are more important in training. We can see (in Fig.~\ref{fig:grad}) that the curves of $\alpha$-power losses approach that of entropy as $\alpha$ goes to 1. Larger $\alpha$ more possibly pushes the gradients of uncertain samples (with $p$ near 0.5) to zero as shown in Fig.~\ref{fig:grad_small} (also in Fig.~\ref{fig:alpha_power} for three-way classification), and hence neglects their importance in network training. In Fig.~\ref{fig:alpha}, we show that in range $[1,10]$, $\alpha$ larger than 2 achieves better performance than $\alpha$ smaller than 2. This may be because, for $\alpha$ (in $[1,10]$) larger than 2, the uncertain samples that are more likely incorrectly predicted are less important (compared with $\alpha$ smaller than 2) when reducing the prediction uncertainty.
\begin{figure}[t]
	\centering
     \subfigure[] { \includegraphics[width=0.48\columnwidth]{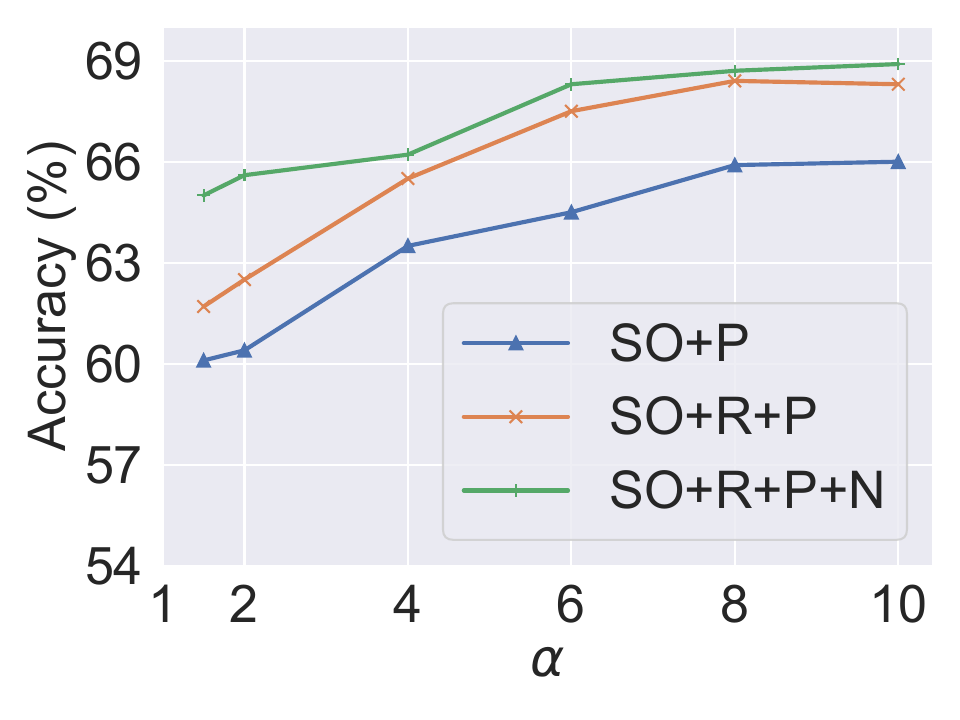}
		\label{fig:alpha} }
	\subfigure[] { \includegraphics[width=0.44\columnwidth]{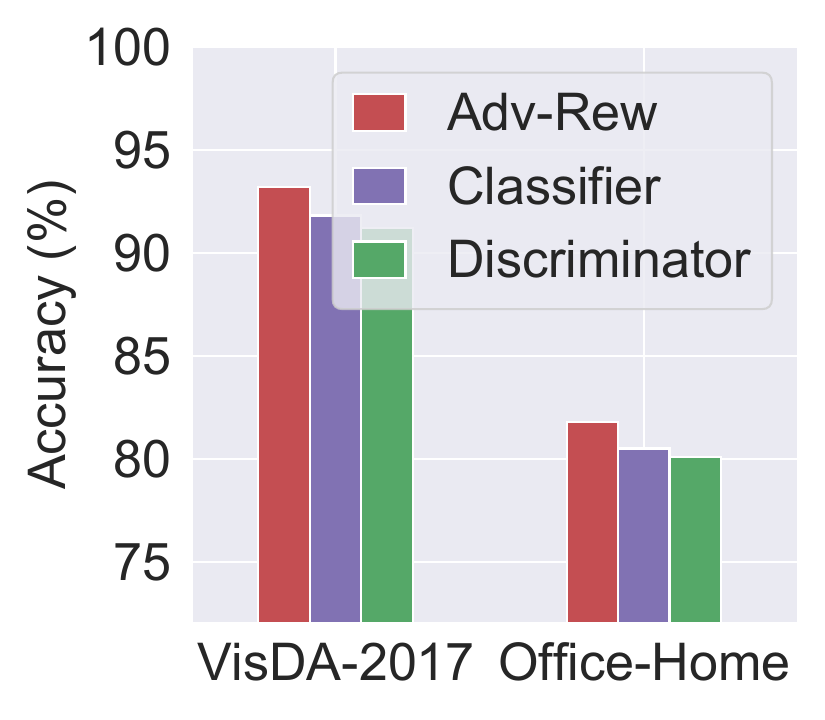}
		\label{fig:comparision_different_weighting} }
	\caption{ (a) Results with varying $\alpha$ in task A$\rightarrow$C on Office-Home dataset. (b) Results of different reweighting strategies on Office-Home and VisDA-2017 datasets. }
\end{figure}

 \vspace{0.5\baselineskip} \noindent \textbf{Comparison of different reweighting strategies.}
 We compare different reweighting strategies for obtaining the weights used in our loss of Eq.~\eqref{eq:class_loss} for PDA, including our adversarial reweighting (Adv-Rew), reweighting based on the classifier in the PDA methods~\cite{cao2018partial,cao2018partial1,li2020deep,liang2020balanced}, and reweighting by the output of discriminator on source data as in~\cite{zhang2018importance}. For the classifier-based strategy, the source data weight of the $k$-th class is defined by $\frac{1}{n}\sum_{j=1}^n C(F(\x_j^t))_k$. For the discriminator-based strategy, we introduce a discriminator $\tilde{D}$ that aims to predict 1 (\textit{resp.} 0) on the target (\textit{resp.} source) domain data. The weight of source domain data $\x_i^s$ is $\tilde{D}(\x_i^s)$.
 The results in Fig.~\ref{fig:comparision_different_weighting} show that our adversarial reweighting outperforms the other two reweighting strategies on VisDA-2017 and Office-Home datasets, confirming the effectiveness of our adversarial reweighting strategy.

 \begin{figure}[t]
	\centering
	\subfigure[] { \includegraphics[width=0.44\columnwidth]{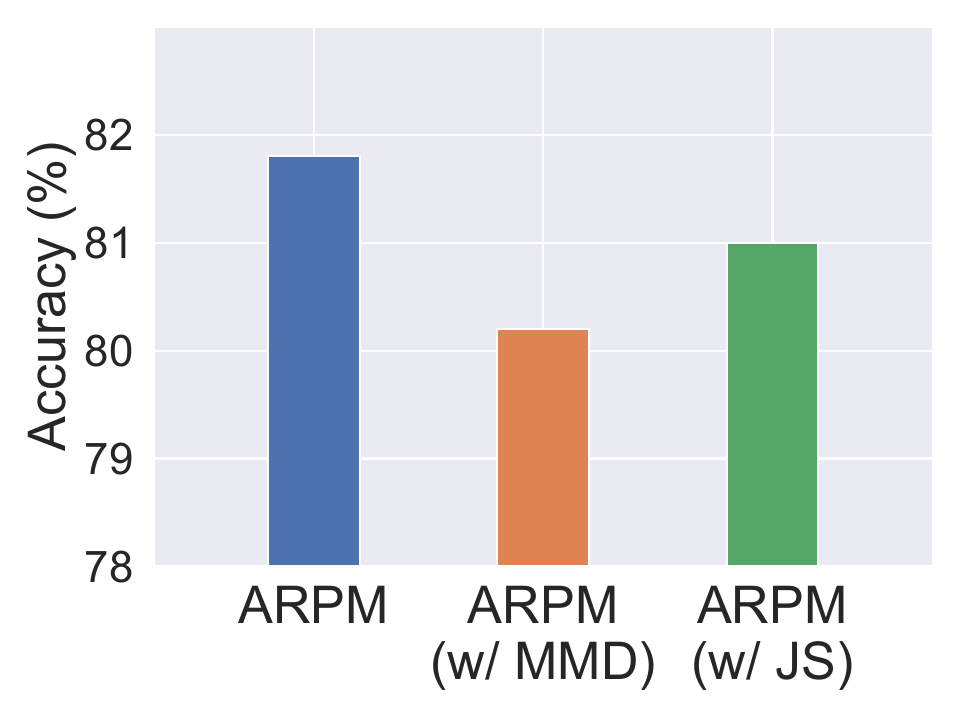}
		\label{fig:comparision_MMD_JS} }
	\subfigure[] { \includegraphics[width=0.44\columnwidth]{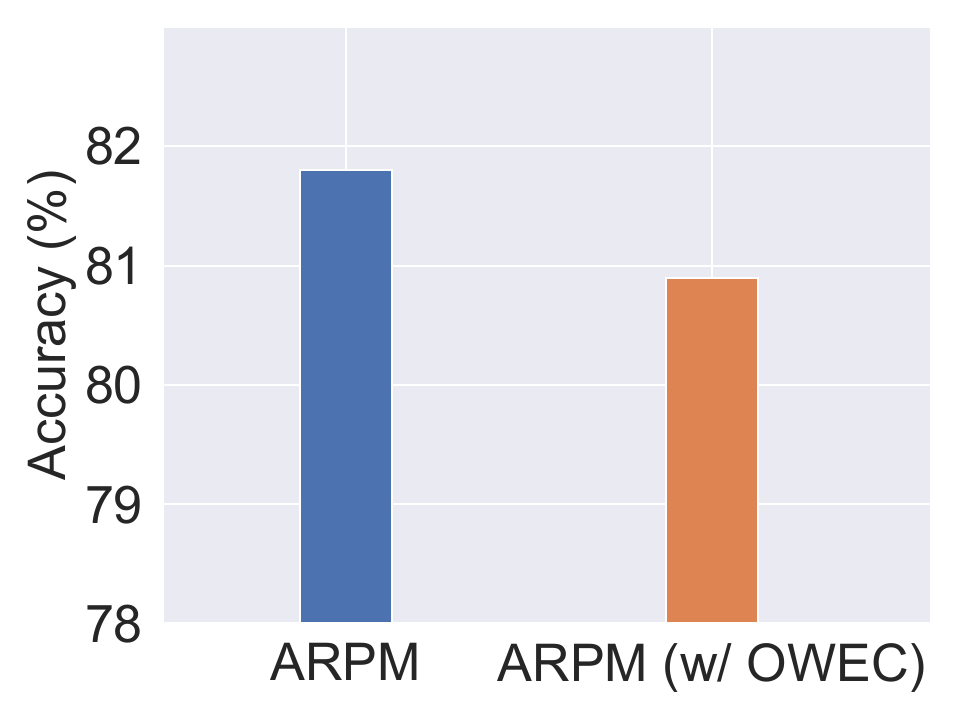}
		\label{fig:comparision_one_weight_each_class}}
	\caption{ (a) Results for MMD and JS-divergence for learning source data weights in our framework on Office-Home dataset. (b) Results for learning one weight for each class (OWEC) on Office-Home dataset. }
\end{figure}

 \vspace{0.5\baselineskip} \noindent \textbf{Comparison with MMD and JS-divergence to learn the weights.}
As discussed in Sect.~\ref{sec:adv_rew}, we minimize the Wasserstein distance to learn the data weights deducing our adversarial reweighting model. We compare ARPM with the approaches that minimize the JS-divergence and Maximum Mean Discrepancy (MMD) to learn the data weights in our framework (denoted as ARPM (w/ JS) and ARPM (w/ MMD), respectively), on Office-Home dataset for PDA. Note that the JS-divergence also induced an adversarial reweighting model where the discriminator is trained with loss in~\cite{goodfellow2014generative}.
In Fig.~\ref{fig:comparision_MMD_JS}, we can see that ARPM using the Wasserstein distance for learning the data weights outperforms ARPM (w/ JS) and ARPM (w/ MMD). The reasons could be as follows.  When the supports of source and target distributions are disjoint, the Wasserstein distance may be more suitable to measure their distance than JS-divergence~\cite{arjovsky2017wasserstein}. The MMD with widely used kernels may be unable to capture very complex distances in high dimensional spaces~\cite{arjovsky2017wasserstein,reddi2015high}, possibly making it less effective than the Wasserstein distance in our framework. 
 
 \vspace{0.5\baselineskip} \noindent \textbf{Comparison with learning one weight for each class (OWEC).}
 The previous PDA methods~\cite{cao2018partial,cao2018partial1,li2020deep,liang2020balanced} assign one weight for each class. As comparisons, we conduct experiments for learning one weight shared by samples of each class in our adversarial reweighting model (denoted as ARPM (w/ OWEC)) on Office-Home dataset for PDA. The results in Fig.~\ref{fig:comparision_one_weight_each_class} show that ARPM with individual weight for each sample outperforms ARPM (w/ OWEC), implying that learning individual weight for each sample is more effective. If the weight is learned for each sample, it is possible to assign higher weights to samples (even in the same class) closer to the target domain. The model trained in such a case may be more transferable, because samples (even in the source domain common classes) less relevant to the target domain become less important.

 \begin{table}[t]
 \caption{Results (under varying random seeds) of gradient penalty (GP) and spectral normalization (SP) to impose Lipschitz constraint on Office-Home dataset.}
     \centering
     \setlength{\tabcolsep}{10pt}	
     \begin{tabular}{l|ccc|cc}
     \toprule
          Seed & 2019&2021&2023&Avg&Std \\
    \midrule
          SP&81.8 &81.6 &82.1 &81.8 &0.3 \\
          GP&  81.1 &80.6 &81.9 &81.2&0.7 \\
    \bottomrule
     \end{tabular}
     \label{tab:gp_sp}
 \end{table}
 
 \vspace{0.5\baselineskip} \noindent \textbf{Comparison of gradient penalty (GP) and spectral normalization (SP) to impose Lipschitz constraint.}
 We compare the gradient penalty (GP) and spectral normalization (SP) to impose Lipschitz constraint in Eq.~\eqref{eq:adversarial_weight_learning_model}. The results in Table~\ref{tab:gp_sp} indicate that spectral normalization achieves better average accuracy and lower standard variation. As discussed in Sect.~\ref{sec:adv_rew}, GP introduces additional hyper-parameter and randomness, which may degrade the reproducibility.
 
 \vspace{0.5\baselineskip} \noindent \textbf{PCA initialization of classifier improving reproducibility.}
 Table~\ref{tab:pca} shows that initializing the weight of the classifier by PCA as discussed in Sect.~\ref{sec:network} does improve the reproducibility of our method (\ie, our method with PCA performs well over different random seeds). Note that $\nabla_{\mathbf{p}}\mathcal{H}_{\alpha}(\mathbf{p}) \propto (p_1^{\alpha-1},p_2^{\alpha-1},\cdots,p_{|\Y|}^{\alpha-1})$, which implies that $\mathbf{p}$ is updated with high possibility towards the class corresponding to the largest element of $\mathbf{p}$. Therefore, good initialization with ``correct'' $\nabla_{\mathbf{p}}\mathcal{H}_{\alpha}(\mathbf{p})$ on target domain may yield better results. Our PCA-based initialization specifies the variance preservation idea of commonly used random initialization strategies~\cite{pmlr-v9-glorot10a,7410480} for normalized features and reduces the randomness, which may yield better reproducibility.
 
\begin{table}[t]
 \caption{Results (under varying random seeds) of ARPM with and without PCA for initializing classifier on VisDA-2017 dataset.}
     \centering
     \setlength{\tabcolsep}{7.5pt}	
     \begin{tabular}{l|ccccc}
     \toprule
          Seed &2015&2017&2019&2021&2023\\
    \midrule
          w/ PCA&92.4&93.5&92.2&93.9&93.6\\
          w/o PCA &93.8&76.8&94.1&78.6& 93.9\\
    \bottomrule
     \end{tabular}
     \label{tab:pca}
 \end{table}
 

 \begin{figure}[t]
	\centering
    \subfigure[]{ \includegraphics[width=0.45\columnwidth]{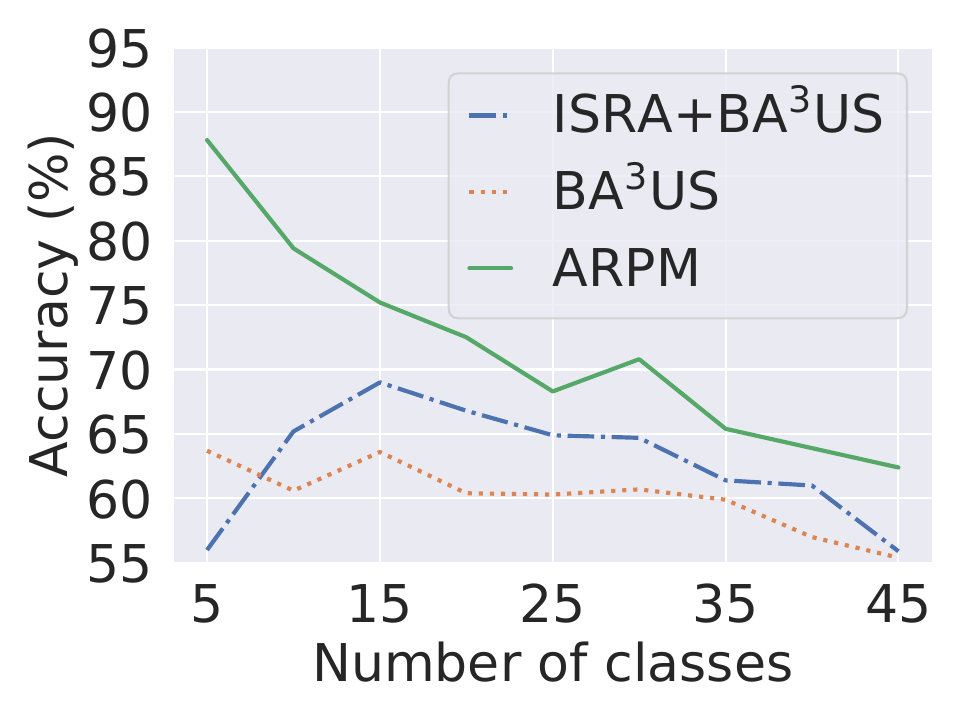}
		\label{fig:num_AC}}
	\subfigure[]{ \includegraphics[width=0.45\columnwidth]{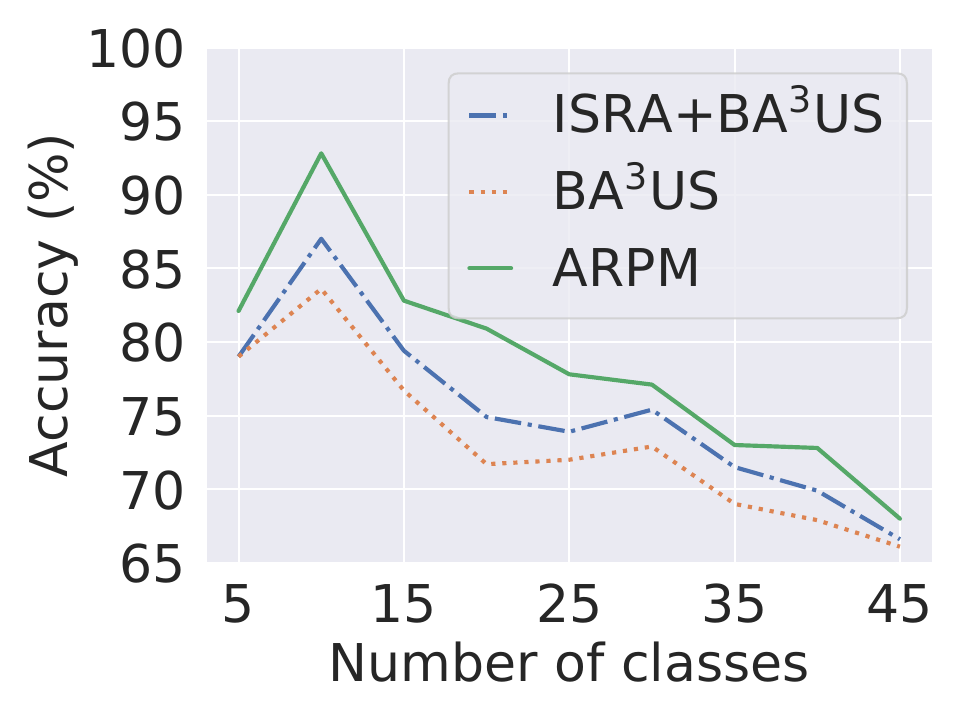} 
		\label{fig:num_CA}}
    
	\caption{ Accuracy with varying numbers of target classes in tasks (a) A$\rightarrow$C and (b) C$\rightarrow$A on Office-Home dataset.
    }
    \label{fig:number_of_classes}
\end{figure}

 \vspace{0.5\baselineskip} \noindent \textbf{Accuracy with varying numbers of target classes.} 
 We evaluate our method with different numbers of target classes in Fig.~\ref{fig:number_of_classes}. Our method of ARPM outperforms recent PDA methods BA$^3$US~\cite{liang2020balanced} and ISRA+BA$^3$US~\cite{xiao2021implicit} when the number of target classes is smaller than 45 (the number of source classes is 65). This indicates that our method is effective for PDA with different degrees of label space mismatch.

 \vspace{0.5\baselineskip} \noindent \textbf{Sensitivity to hyper-parameters.}
  We investigate the effect of hyper-parameters  $\lambda$ and $\rho$ in Fig.~\ref{fig:hyper_param}. Our method is relatively stable to $\rho$ in range [3, 9] as shown in Fig.~\ref{fig:rho} and to $\lambda$ in [0.3, 0.7] as shown in Fig.~\ref{fig:lam}.

  \begin{figure}[t]
	\centering
    \subfigure[]{ \includegraphics[width=0.45\columnwidth]{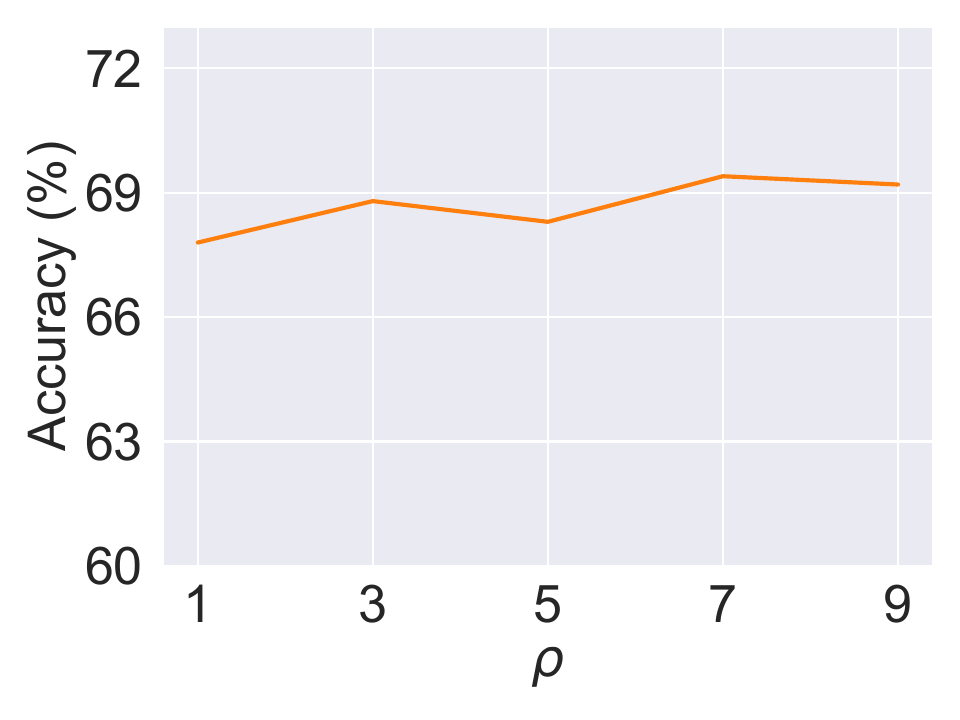}
		\label{fig:rho}}
	\subfigure[]{ \includegraphics[width=0.45\columnwidth]{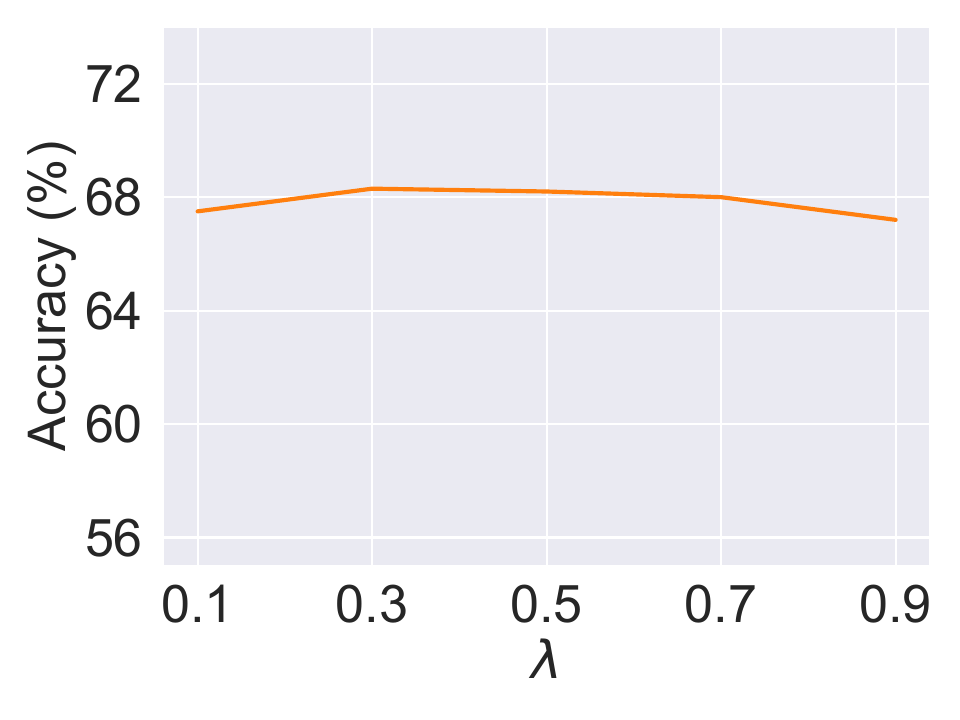} 
		\label{fig:lam}}
    
	\caption{ Results for varying magnitudes of (a) $\rho$ in the constraints of adversarial reweighting model and (b) $\lambda$ in Eq.~\eqref{eq:overall_loss}, in task A$\rightarrow$C on Office-Home dataset.
    }
    \label{fig:hyper_param}
\end{figure}

 \begin{figure}[t]
	\centering
    \subfigure[]{ \includegraphics[width=0.46\columnwidth]{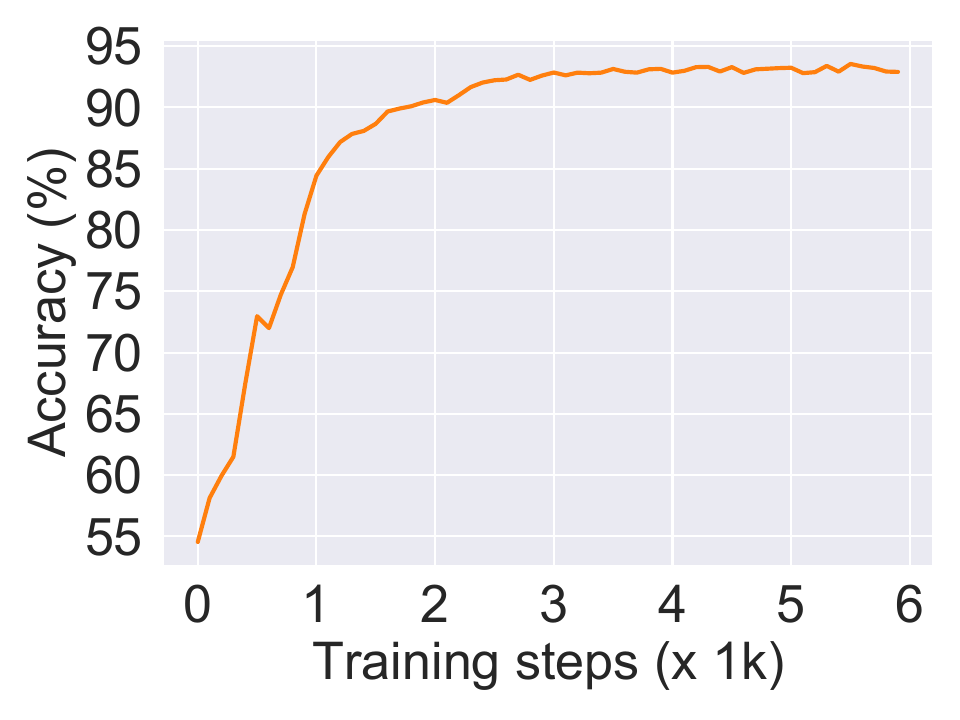} 
		\label{fig:acc}}
	\subfigure[]{ \includegraphics[width=0.46\columnwidth]{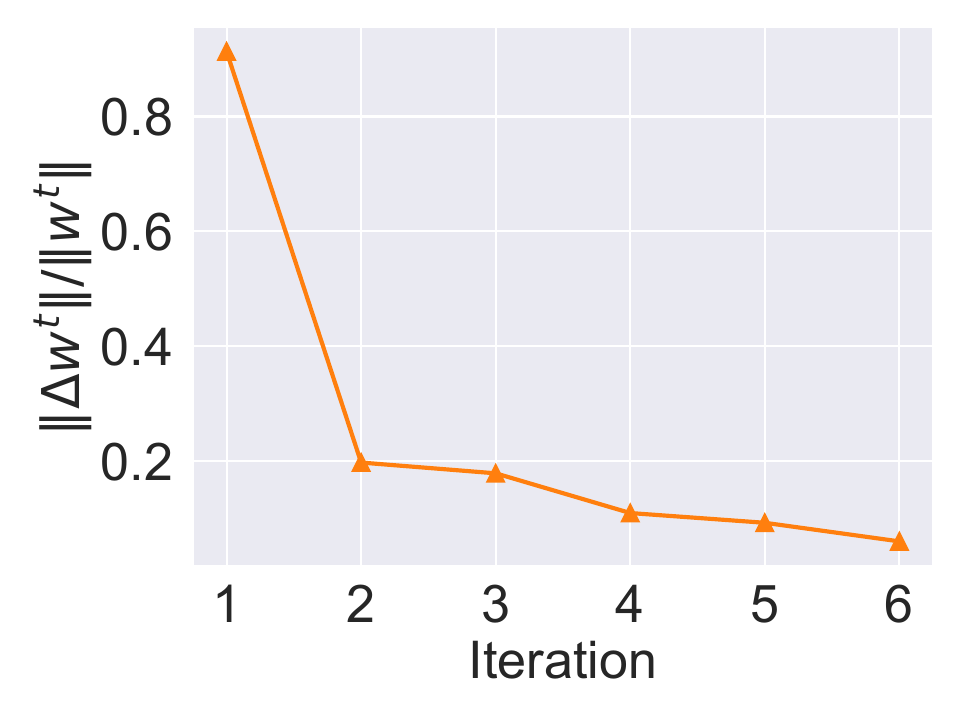} 
		\label{fig:delta_weight}}
    
	\caption{ (a) Accuracy in training in task S$\rightarrow$R on VisDA-2017 dataset. (b) {Relative} difference of source data weights in alternate iteration in task S$\rightarrow$R on VisDA-2017 dataset.
    }
    \label{fig:convergence}
\end{figure}

 \vspace{0.5\baselineskip} \noindent \textbf{Convergence of training algorithm.}
In Fig.~\ref{fig:convergence}, we take the PDA task S$\rightarrow$R on VisDA-2017 as an example to study the convergence of our method. Figure~\ref{fig:acc} indicates that the accuracy of our approach stably increases and converges in the training process.  We also show the relative difference of weights in Fig.~\ref{fig:delta_weight}. The relative difference is $\frac{\|\Delta\mathbf{w}^t\|}{\|\mathbf{w}^t\|}$, where $\Delta\mathbf{w}^t=\mathbf{w}^{t+1}-\mathbf{w}^t$, and $\mathbf{w}^t$ is the value of the weights in the $t$-th iteration of the alternate training algorithm. We can see that $\frac{\|\Delta\mathbf{w}^t\|}{\|\mathbf{w}^t\|}$ stably decreases.

 \vspace{0.5\baselineskip} \noindent \textbf{Visualization of learned weights.}
 We visualize the learned average weights of source domain data of each class in task S$\rightarrow$R on VisDA-2017 dataset, as shown in Fig.~\ref{fig:hist_weights}. We can see that the source domain common class (the first six classes) data are assigned with larger weights in general (except the 6-th class). Even for the 6-th class, its weight is larger than the weights of five among the total of six source-private classes.
  \begin{figure}[t]
	\centering
    \subfigure[]{ \includegraphics[width=0.45\columnwidth]{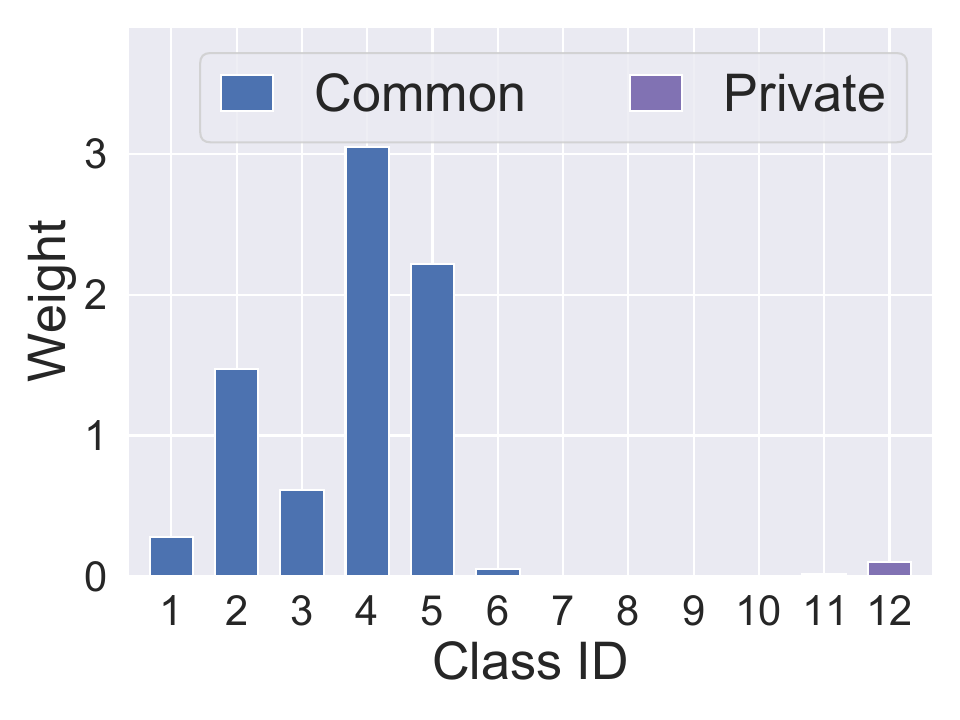} 
		\label{fig:hist_weights}}
	\subfigure[]{ \includegraphics[width=0.45\columnwidth]{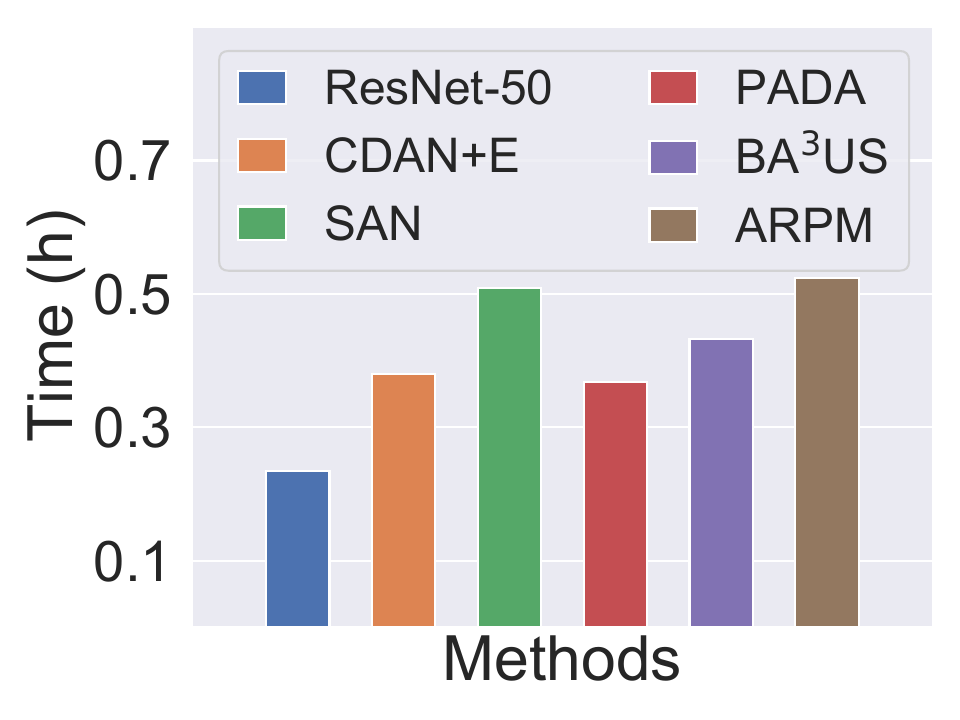} 
		\label{fig:computational_cost}}
    
	\caption{ (a) Average weights for each class on source domain in task S$\rightarrow$R on VisDA-2017. (b) Computational cost of PDA methods in task A$\rightarrow$C on Office-Home.
    }
\end{figure}

\begin{figure}[t]
	\centering
        \includegraphics[width=1.02\columnwidth]{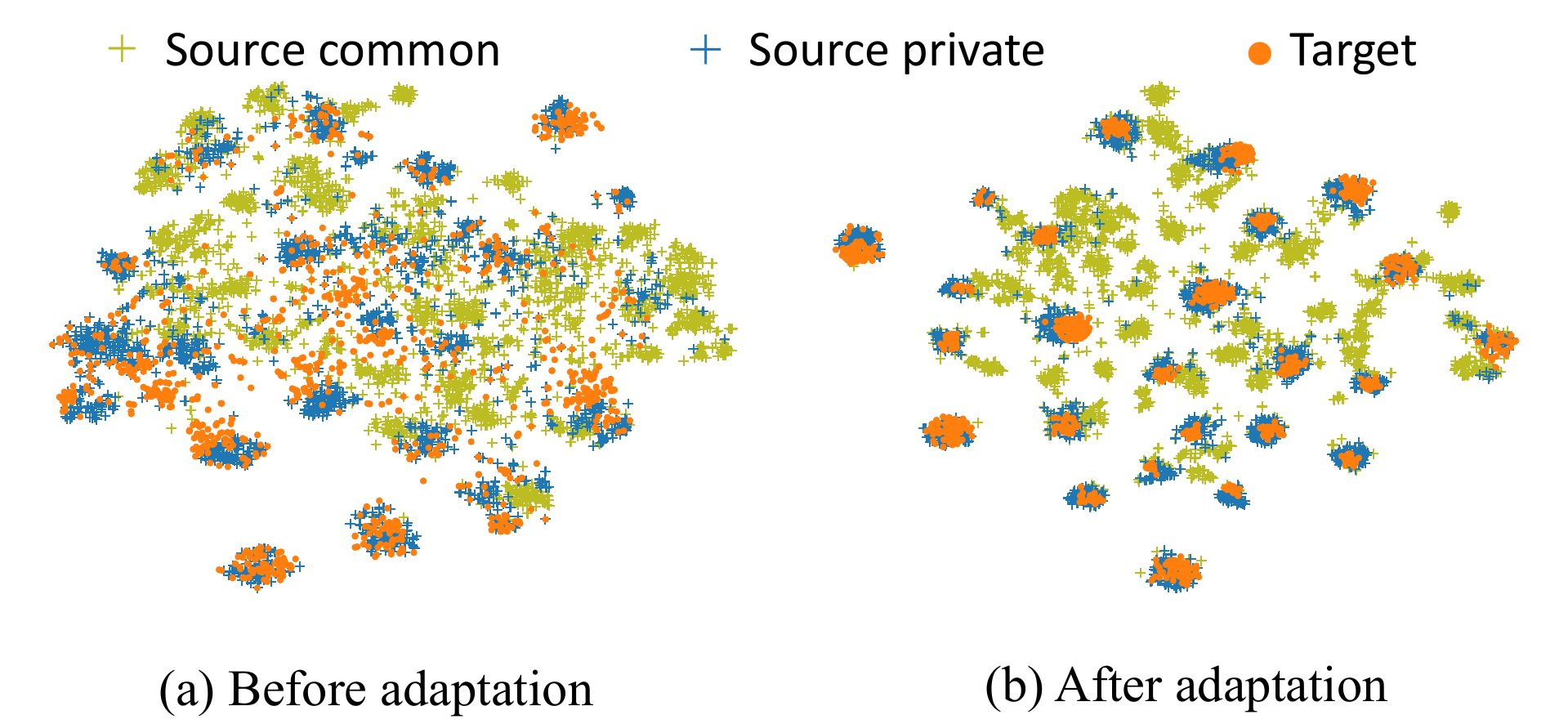}
   
	\caption{T-SNE visualization of features in task R$\rightarrow$A on Office-Home dataset. (a) Features before adaptation. (b) Features after adaptation of ARPM.}
 \label{fig:tsne_features}
\end{figure}

 \vspace{0.5\baselineskip} \noindent \textbf{Computational cost.}
 We compare the computational cost of different methods with the total training time in the same training steps (5000 steps), as in Fig.~\ref{fig:computational_cost}. Figure~\ref{fig:computational_cost} shows that our approach (ARPM) is comparable to other methods in terms of computational time cost. Note that the test time cost is similar for all the methods because all the methods only need one forward process.

{
 \vspace{0.5\baselineskip} \noindent \textbf{Feature visualization.}
 We visualize the features before and after the adaptation of ARPM in Fig.~\ref{fig:tsne_features} by T-SNE~\cite{van2008visualizing}. We can see that from Fig.~\ref{fig:tsne_features}(b), the source domain common class features are more discriminative/separable than the source-private class features, and the target domain data features are aligned with the source domain common class features.
}

\section{Conclusion}
In this paper, we propose a novel ARPM approach for PDA, in which we propose an adversarial reweighting model to learn to reweight source domain data, propose $\alpha$-power maximization to reduce prediction uncertainty, and utilize the neighborhood reciprocity clustering to enforce robustness. Extensive experiments on five benchmark datasets demonstrate the effectiveness of the proposed ARPM for PDA. We also present the theoretical analysis of the proposed method. 
{Additionally, we extend our approach to more ``open-world'' recognition tasks, including open-set DA, universal DA, and TTA. Since both the adversarial reweighting model and the $\alpha$-power maximization in our approach require accessing the target domain data, it is non-trivial to extend our approach to the adaptation tasks without target domain data, \eg, domain generalization. We will explore more applications of our approach in the future.}


\section{Acknowledgments}
The work was supported by National Key R\&D Program 2021YFA1003002, NSFC (12125104, U20B2075, 12326615), Postdoctoral Fellowship Program of CPSF GZB20230582, and Key Laboratory of Biomedical Imaging Science and System, Chinese Academy of Sciences.

\section{Data Availability Statement}
The data that support the findings of this study
are available from the authors upon
request.

\section{Statements and Declarations}
The authors declare that they have no known competing financial interests or personal relationships
that could have appeared to influence the work
reported in this paper.

\section*{Appendix}\label{sec:app}

We first give some lemmas and then provide the proof of Theorem~\ref{thm:pda_bound}.

\begin{lemma} \label{thm:lemma_R}
    Divide $\Yc$ into $S_1$ and $S_2$ such that $S_1 = \{i\in\Yc:\E_{(\x,y)\sim\frac{\tP_i+Q_i}{2}}\mathbb{I}(\exists\x'\in\mathcal{N}(\x),\tilde{f}(\x)\neq\tilde{f}(\x'))<\min\{\epsilon,q\}\}$ and $S_2 = \{i\in\Yc:\E_{(\x,y)\sim\frac{\tP_i+Q_i}{2}}\mathbb{I}(\exists\x'\in\mathcal{N}(\x),\tilde{f}(\x)\neq\tilde{f}(\x'))\geq\min\{\epsilon,q\}\}$. Under the condition of Theorem~\ref{thm:pda_bound}, we have 
    \begin{equation}
        \sum_{i\in S_2} \frac{\tP+Q}{2}(y=i) \leq \frac{R_{\frac{\tP+Q}{2}}(f)}{\min \{\epsilon,q\}}.
    \end{equation}
\end{lemma}

 \begin{proof}
    Suppose $\sum_{i\in S_2} \frac{\tP+Q}{2}(y=i) > \frac{R(f)}{\min \{\epsilon,q\}}$, which implies 
    \begin{equation}
        \begin{split}
           & R_{\frac{\tP+Q}{2}}(f) \\=& \frac{\tP+Q}{2}(\{\x,\exists\x'\in\mathcal{N}(\x),\tilde{f}(\x)\neq\tilde{f}(\x')\})\\
            =&\E_{(\x,y)\sim\frac{\tP+Q}{2}}\mathbb{I}(\exists\x'\in\mathcal{N}(\x),\tilde{f}(\x)\neq\tilde{f}(\x'))\\
            =& \sum_{i\in\Yc}\Big\{\E_{\x\sim\frac{\tP_i+Q_i}{2}}\mathbb{I}(\exists\x'\in\mathcal{N}(\x),\tilde{f}(\x)\neq\\&\hspace{1cm}\tilde{f}(\x'))\frac{\tP+Q}{2}(y=i)\Big\}\\
            \geq& \sum_{i\in S_2}\Big\{\E_{\x\sim\frac{\tP_i+Q_i}{2}}\mathbb{I}(\exists\x'\in\mathcal{N}(\x),\tilde{f}(\x)\neq \\&\hspace{1cm}\tilde{f}(\x'))\frac{\tP+Q}{2}(y=i)\Big\}\\
            \geq & \min \{\epsilon,q\} \sum_{i\in S_2} \frac{\tP+Q}{2}(y=i)
            > R_{\frac{\tP+Q}{2}}(f).
        \end{split}
    \end{equation}
    $R_{\frac{\tP+Q}{2}}(f)>R_{\frac{\tP+Q}{2}}(f)$ forms a contradiction.
 \end{proof}

 \begin{lemma}[Lemma 2 in \cite{liu2021cycle}]\label{thm:lemma_q}
    Under the condition of Theorem~\ref{thm:pda_bound}, if sub-populations $\tP_i$ and $Q_i$ satisfy $\E_{(\x,y)\sim\frac{\tP_i+Q_i}{2}}\mathbb{I}(\exists\x'\in\mathcal{N}(\x),\tilde{f}(\x)\neq\tilde{f}(\x'))<\min\{\epsilon,q\}$, we have 
    \begin{equation}
        |\varepsilon_{\tP_i}(f)-\varepsilon_{Q_i}(f)|\leq 2q.
    \end{equation}
 \end{lemma}

 \begin{lemma}[Lemma 3 in \cite{liu2021cycle}]\label{thm:lemma_M}
    For any distribution $P$, if $f$ is $L$-Lipschiz w.r.t. $d(\cdot,\cdot)$, we have 
    \begin{equation}
        R_P(f) \leq \frac{1}{(1-2L\xi)}(1-{M}_P(f)).
    \end{equation}
 \end{lemma}

 \vspace{0.5\baselineskip} \noindent \textbf{Proof of Theorem~\ref{thm:pda_bound}.}
 From the definition of $\varepsilon_Q(f)$ in PDA, we have
 \begin{equation}\label{eq:app_eq1}
    \begin{split}
        \varepsilon_Q(f) = & \sum_{i\in\Yc}\varepsilon_{Q_i}(f)Q(y=i) \\
        \leq & \sum_{i\in S_1}\varepsilon_{Q_i}(f)Q(y=i)  + \sum_{i\in S_2} Q(y=i)\\
        \leq & \sum_{i\in S_1}(\varepsilon_{P_i}(f)+2q)r\tP(y=i) \\
        &+ \sum_{i\in S_2} Q(y=i)\\
        \leq & \sum_{i\in \Yc}(\varepsilon_{P_i}(f)+2q)r\tP(y=i) \\
        &+ \sum_{i\in S_2} Q(y=i)\\
        = & r\varepsilon_{\tP}(f) + 2qr + \sum_{i\in S_2} Q(y=i).
    \end{split}
 \end{equation}
 The second inequality uses Lemma~\ref{thm:lemma_q} and $\frac{Q(y=i)}{\tP(y=i)}\leq r$ for $i\in\Yc$. 
 Since
 \begin{equation}
    \begin{split}
         \frac{\tP+Q}{2}(y=i) =& \frac{1}{2}(\tP(y=i)+Q(y=i))\\
         \geq&  \frac{1}{2}(\frac{1}{r}Q(y=i)+Q(y=i)) \\
         = & \frac{1+r}{2r}Q(y=i),
    \end{split}
 \end{equation}
 we have  
 \begin{equation}
     \sum_{i\in S_2} Q(y=i) \leq  \frac{2r}{1+r}\sum_{i\in S_2} \frac{\tP+Q}{2}(y=i).
 \end{equation}
 Using Lemma~\ref{thm:lemma_R}, we have
 \begin{equation}
    \begin{split}
        \sum_{i\in S_2} Q(y=i) 
        \leq &\frac{2r}{\min \{\epsilon,q\}(1+r)}R_{\frac{\tP+Q}{2}}(f).
    \end{split}
 \end{equation}
 Applying Lemma~\ref{thm:lemma_M}, for any $\eta\in[0,1]$, we have 
 \begin{equation}\label{eq:app_eq2}
     \begin{split}
        &\sum_{i\in S_2} Q(y=i) \\
        \leq &\frac{2r\eta}{\min \{\epsilon,q\}(1+r)}R_{\frac{\tP+Q}{2}}(f)\\
        +& \frac{2r(1-\eta)}{\min \{\epsilon,q\}(1+r)(1-2L\xi)}(1-\mathcal{M}_{\frac{\tP+Q}{2}}(f)).
    \end{split}
 \end{equation}
 Combining Eqs.~\eqref{eq:app_eq1} and~\eqref{eq:app_eq2}, we have
 \begin{equation}
    \begin{split}
        \varepsilon_Q(f) \leq &  r\varepsilon_{\tP}(f) + c_1 R_{\frac{\tP+Q}{2}}(f)\\
        &+c_2(1-{M}_{\frac{\tP+Q}{2}}(f)) + 2rq,
    \end{split}
\end{equation}
where the coeffcients $c_1 = \frac{2\eta r}{\min\{\epsilon,q\}(1+r)}$ and $c_2 = \frac{2r(1-\eta) }{\min\{\epsilon,q\}(1-2L\xi)(1+r)}$ are constants to $f$. \qed


\bibliography{ref}

\end{document}